\documentclass{article}

\usepackage{microtype}
\usepackage{graphicx}
\usepackage{subcaption}
\usepackage{booktabs}

\usepackage{hyperref}

\usepackage{babel}
\usepackage{microtype}

\usepackage{amsmath}
\usepackage{amsthm}
\usepackage{amssymb}
\usepackage{booktabs}
\usepackage{dsfont}
\usepackage{graphicx}
\usepackage{here}
\usepackage{microtype}
\usepackage{mleftright}
\usepackage{soul}
\usepackage{wrapfig}

\usepackage{tikz}

\usepackage{url}
\usepackage{paralist}
\usepackage{siunitx}
 \usepackage{subcaption}
\usepackage{thm-restate}
\usepackage{xspace}
\usepackage{multirow}
\usepackage{thm-restate}

\usepackage{svg}
\graphicspath{{figures/}}

\hypersetup{
  colorlinks,
urlcolor  = blue,
  citecolor = blue,
  linkcolor = blue,
}

\definecolor{DAEgreen}{RGB}{196, 227, 180}
\definecolor{lightgrey}{RGB}{230,230,230}
\definecolor{DAEred}{HTML}{e31a1c}

\DeclareRobustCommand\framebox[1]{\textcolor{#1}{}\begin{tikzpicture}\node at (0,0) [color=#1,line width=1, rectangle, draw] {}; \end{tikzpicture}}

\PassOptionsToPackage{sort}{natbib}

\usepackage[accepted]{icml2024}

\usepackage{amsmath}
\usepackage{amssymb}
\usepackage{mathtools}
\usepackage{amsthm}
\usepackage{algorithm}
\usepackage{nicefrac}
\usepackage[export]{adjustbox}
\usepackage{balance}
\usepackage[capitalize]{cleveref}

\theoremstyle{plain}
\newtheorem{theorem}{Theorem}[section]

\newtheorem{lemma}[theorem]{Lemma}

\theoremstyle{definition}
\newtheorem{definition}[theorem]{Definition}

\theoremstyle{remark}

\usepackage[textsize=tiny]{todonotes}

\icmltitlerunning{Mapping the Multiverse of Latent Representations}

\newcommand{\ourmethod}{\textsc{Presto}\xspace}
\newcommand{\ourdistance}{PD\xspace}
\newcommand{\ourvariance}{PV\xspace}
\newcommand{\oursensitivity}{PS\xspace}
\newcommand{\ourcoordinatesensitivity}{PS\xspace}
\newcommand{\distancenorm}{\ensuremath{p}\xspace}
\newcommand{\cardinality}[1]{\ensuremath{|#1|}\xspace}
\newcommand{\probability}{\ensuremath{p}\xspace}
\newcommand{\LVAE}{\ensuremath{L_{\text{VAE}}}\xspace}
\newcommand{\fred}{\ensuremath{f_{\text{red}}}\xspace}
\newcommand{\reals}{\ensuremath{\mathds{R}}\xspace}
\newcommand{\naturals}{\ensuremath{\mathds{N}}\xspace}
\newcommand{\Dimension}{\ensuremath{D}\xspace}
\newcommand{\dimension}{\ensuremath{d}\xspace}
\newcommand{\nneighbors}{\ensuremath{k}\xspace}
\newcommand{\topospace}{\ensuremath{X}\xspace}
\newcommand{\topodim}{\ensuremath{h}\xspace}
\newcommand{\Hom}{\ensuremath{H}\xspace}
\newcommand{\filtrationparam}{\ensuremath{t}\xspace}
\newcommand{\dgm}{\ensuremath{\text{dgm}}\xspace}
\newcommand{\dgms}{\ensuremath{\text{Dgm}}\xspace}
\newcommand{\landscape}{\ensuremath{L}\xspace}
\newcommand{\nprojections}{\ensuremath{\pi}\xspace}
\newcommand{\nsamples}{\ensuremath{s}\xspace}
\newcommand{\nlandscapes}{\ensuremath{N}\xspace}
\newcommand{\landscapes}{\ensuremath{\mathcal{L}}\xspace}
\newcommand{\representatives}{\ensuremath{\mathcal{R}}\xspace}
\newcommand{\harmonic}{\ensuremath{\mathcal{H}}\xspace}
\newcommand{\proxy}{\ensuremath{P}\xspace}
\newcommand{\ncover}{\ensuremath{c^\ast}\xspace}
\newcommand{\BO}{\ensuremath{\mathcal{O}}\xspace}
\newcommand{\tildeO}{\ensuremath{\widetilde{\mathcal{O}}}\xspace}
\newcommand{\eqclass}{\ensuremath{Q}\xspace}
\newcommand{\nparams}{\ensuremath{c}\xspace}
\newcommand{\neqclasses}{\ensuremath{q}\xspace}
\newcommand{\eqclasses}{\ensuremath{\mathcal{Q}}\xspace}
\newcommand{\algos}{\ensuremath{\mathcal{A}}\xspace}
\newcommand{\implementations}{\ensuremath{\mathcal{I}}\xspace}
\newcommand{\implementation}{\ensuremath{\iota}\xspace}
\newcommand{\datas}{\ensuremath{\mathcal{D}}\xspace}
\newcommand{\data}{\ensuremath{\delta}\xspace}
\newcommand{\algo}{\ensuremath{\alpha}\xspace}
\newcommand{\dataset}{\ensuremath{X}\xspace}

\newcommand{\params}{\ensuremath{\theta}\xspace}

\newcommand{\multiverse}{\ensuremath{\mathcal{M}}\xspace}
\newcommand{\MMS}{\ensuremath{\mathfrak{M}}\xspace}

\newcommand{\probe}{\ensuremath{M}\xspace}
\newcommand{\model}{\ensuremath{M}\xspace}
\newcommand{\anydimension}{\ensuremath{\ast}\xspace}
\newcommand{\embeddings}{\ensuremath{\mathcal{E}}\xspace}
\newcommand{\projections}{\ensuremath{\mathcal{P}}\xspace}

\newcommand{\nconfigs}{\ensuremath{m}\xspace}

\newcommand{\embedding}{\ensuremath{E}\xspace}
\newcommand{\projection}{\ensuremath{P}\xspace}
\newcommand{\projector}{\ensuremath{f}\xspace}
\newcommand{\distance}{\ensuremath{d}\xspace}

\newcommand{\topologybound}{\ensuremath{\ell^k}\xspace}

\newcommand{\projdim}{\ensuremath{k}\xspace}
\newcommand{\topological}{\ensuremath{T}\xspace}

\newcommand{\vae}{\textsc{VAE}\xspace}
\newcommand{\bvae}{\ensuremath{\beta}-\vae}
\newcommand{\info}{\textsc{InfoVAE}\xspace}
\newcommand{\wae}{\textsc{WAE}\xspace}
\newcommand{\mnist}{\textit{MNIST}\xspace}
\newcommand{\cifar}{\textit{CIFAR-10}\xspace}
\newcommand{\celeba}{\textit{celebA}\xspace}
\newcommand{\dsprites}{\textit{dsprites}\xspace}
\newcommand{\fashionmnist}{\textit{FashionMNIST}\xspace}

\newcommand{\ada}{\textsc{ada}\xspace}
\newcommand{\mistral}{\textsc{mistral}\xspace}
\newcommand{\distilroberta}{\textsc{distilroberta}\xspace}
\newcommand{\MiniLM}{\textsc{MiniLM}\xspace}
\newcommand{\mpnet}{\textsc{mpnet}\xspace}
\newcommand{\qadistilbert}{\textsc{qa-distilbert}\xspace}
\newcommand{\arxiv}{\textit{arXiv}\xspace}
\newcommand{\bbc}{\textit{bbc}\xspace}
\newcommand{\cnn}{\textit{cnn}\xspace}
\newcommand{\patents}{\textit{patents}\xspace}

\newcommand{\umap}{\textsc{umap}\xspace}
\newcommand{\tsne}{\textsc{t-SNE}\xspace}
\newcommand{\phate}{\textsc{phate}\xspace}
\newcommand{\subspaces}{\mathcal{S}}
\newcommand{\subspace}{\probe(S)}
\newcommand{\genimage}{\ensuremath{I_{\tilde{v}}}\xspace}
\newcommand{\genimages}{\ensuremath{\tilde{I}}\xspace}

\DeclareMathOperator{\diam}{diam}
 
\newcommand{\oururl}{\href{https://doi.org/10.5281/zenodo.11355446}{https://doi.org/10.5281/zenodo.11355446}}

\begin{document}

\twocolumn[
\icmltitle{Mapping the Multiverse of Latent Representations}

\icmlsetsymbol{equal}{*}
\icmlsetsymbol{super}{$\dagger$}

\begin{icmlauthorlist}
  \icmlauthor{Jeremy Wayland}{HM,TUM}
  \icmlauthor{Corinna Coupette}{super,KTH,MPI}
  \icmlauthor{Bastian Rieck}{super,HM,TUM}
\end{icmlauthorlist}

\icmlaffiliation{KTH}{KTH Royal Institute of Technology}
\icmlaffiliation{MPI}{Max Planck Institute for Informatics}
\icmlaffiliation{HM}{Helmholtz Munich}
\icmlaffiliation{TUM}{Technical University of Munich}

\icmlcorrespondingauthor{Jeremy Wayland}{jeremy.wayland@tum.de}
\icmlcorrespondingauthor{Corinna Coupette}{coupette@mpi-inf.mpg.de}
\icmlcorrespondingauthor{Bastian Rieck}{bastian.rieck@tum.de}

\icmlkeywords{Machine Learning, Generative Models, Manifold Learning, Persistent Homology, ICML}

\vskip 0.3in
]

\hyphenation{hy-per-pa-ram-e-ter hy-per-pa-ram-e-ters rep-re-sen-ta-tion rep-re-sen-ta-tions rep-re-sen-ta-tion-al over-fit-ting pre-dic-tive con-fig-u-ra-tions demon-strate re-spon-si-ble ag-gre-ga-tion sum-ma-rized Thi-ya-ga-lin-gam em-bed-ded in-di-rect-ly mul-ti-verse mul-ti-ver-ses per-sis-tence al-go-rith-mic par-tic-u-lar topolo-gies com-pu-ta-tion-al-ly pa-ram-e-ter pa-ram-e-ters fea-tures the-o-ret-i-cal-ly sen-si-tiv-i-ty em-bed-ding di-am-e-ter di-men-sions ex-pec-ted in-tro-duced de-scrip-tors em-bed-dings im-ple-men-ta-tion iden-ti-cal ar-chi-tec-tures ex-per-i-ments fol-low-ing dis-en-tan-gle-ment char-ac-ter-is-tics dis-tin-guish-ing ex-plic-it cor-re-la-tions wide-spread prin-ci-pled ho-mol-o-gy de-scrip-tors topo-lo-gi-cal per-sis-tent land-scape re-search ex-plic-it-ly there-fore pro-jec-tions}

\renewcommand{\icmlEqualContribution}{\textsuperscript{$\dagger$}These authors jointly directed this work.}

\printAffiliationsAndNotice{\icmlEqualContribution}

\begin{abstract}
Echoing recent calls to counter reliability and robustness concerns in machine learning via \emph{multiverse analysis}, 
we present \ourmethod, a principled framework for \emph{mapping the multiverse} of machine-learning models that rely on \emph{latent representations}. 
Although such models enjoy widespread adoption, 
the variability in their embeddings remains poorly understood, 
resulting in unnecessary complexity and untrustworthy representations.  
Our framework uses \emph{persistent homology} to characterize the latent spaces arising from different combinations of diverse machine-learning methods, (hyper)parameter configurations, and datasets, 
allowing us to measure their pairwise \emph{(dis)similarity} and statistically reason about their \emph{distributions}.
As we~demonstrate both theoretically and empirically, 
our pipeline preserves desirable properties of collections of latent representations, 
and it can be leveraged to perform sensitivity analysis, 
detect anomalous embeddings, or
efficiently and effectively navigate hyperparameter search~spaces.

 \end{abstract}

\section{Introduction}
\label{Introduction}

Our ability to design and deploy new machine-learning models has far
outpaced our understanding of their inner workings.
The real-world successes of Variational Auto-Encoders (VAEs), 
Large Language Models~(LLMs), 
and Graph Neural Networks~(GNNs) notwithstanding, 
our benchmark-driven engineering approaches 
often come at the cost of an  
\emph{inability to make formal predictions} about
the capacity of a specific model to perform a specific task on a specific dataset.
Thus, when observing a particular performance result, 
we are unsure to which extent it is impacted by 
\begin{inparaenum}[(i)]
  \item intrinsic problems with the data, such as poor data quality or data leakage, 
\item intrinsic problems with the model architecture, such as
    insufficient learning capacity, 
\item misfit between the data and the model architecture,  
\item unsuitable (parameter) choices for the training process, or
\item an under-explored hyperparameter landscape.
\end{inparaenum}
These uncertainties~contribute to a looming \emph{reproducibility crisis} in
machine learning that threatens to impede fundamental progress and reduce real-life
impact~\citep{Haibe-Kains20a, McDermott21a,gundersen2022sources,kapoor2023leakage}.

\textbf{Into the Multiverse.}\quad
Ensuring robust, reliable, and reproducible results in machine-learning applications requires new conceptual frameworks and tools.
As a starting point, we must acknowledge that all data work involves
\emph{many different choices} that may support \emph{many different
conclusions}~\citep{simonsohn2020specification}. 
In machine learning, 
such choices regularly include the model architecture and its
hyperparameters~\citep{feurer2019hyperparameter}, 
the dataset and its preprocessing~\citep{muller2022forgetting}, 
as well as the technicalities of model training and evaluation~\citep{sivaprasad2020optimizer}. 
To rigorously assess machine-learning models, then, 
we should explicitly \emph{embrace the variation} resulting from all reasonable combinations of reasonable choices, 
rather than keep individual choices hidden or implicit. 
This is the essence of \emph{multiverse analysis}~\citep{Steegen16a}, 
which was originally proposed to mitigate the perceived replication
crisis in psychology~\citep{simmons2011false}. 

\textbf{Representation Matters.}\quad 
In multiverse analysis, 
each set of mutually compatible choices (according to some specification) gives rise to a different analytical \emph{universe}, 
and we assess the results of all universes in the same \emph{multiverse} collectively to derive our conclusions.
While the early applications of multiverse approaches in machine
learning~\citep{bell2022modeling,simson2023using} have focused on variability in performance (\emph{performance variability}), 
which afflicts all machine-learning models, 
the influential class of \emph{latent-space models} 
(including VAEs, LLMs, and GNNs) 
also exhibits variability in latent representations  (\emph{representational variability}).  
In many cases, even~relatively~small (hyper)parameter changes can radically alter the embedding structure of latent-space models. 
As an example, consider \Cref{fig:dae}, 
which depicts two-dimensional representations of the \texttt{XYC} dataset 
as generated by Disentangling Auto-Encoders (DAEs) with different learning rates and batch-normalization parameters. 
Although DAEs were recently designed by \citet{cha_orthogonality-enforced_2023} precisely to learn disentangled representations, 
and the \texttt{XYC} dataset was introduced to highlight the power of DAEs,
the latent structure of the \texttt{XYC} dataset is disentangled only with the right parameter choices  (\framebox{DAEred}). 
More generally, representational variability in latent-space models remains poorly understood.

\textbf{Variability Matters.}\quad 
While the \emph{performance} variability of latent-space models is clearly connected to their \emph{reliability}, 
the \emph{representational} variability of such models is directly linked to their \emph{interpretability} and \emph{robustness}.
First, if models differing only in their (hyper)parameters yield similar
performance based on very dissimilar latent spaces, 
we cannot use these latent spaces to understand the models (impairing their interpretability). 
Second, if a small change in the (hyper)parameter or training-data configuration induces a large change in the latent-space structure of a model, 
the model associated with the original configuration may not capture the essence of the task, 
even if the model appears to be competitive when assessed based on
performance-driven evaluation (indicating a lack of structural robustness). 
Therefore, 
representational variability not only complements performance variability in the analysis of latent-space models, 
but \emph{ceteris paribus}, 
models with lower representational variability should be preferred over models with higher representational variability. 
Hence, understanding representational variability in latent-space models is crucial to ensure their overall alignment with responsible-machine-learning goals. 

\begin{figure}[t]
	\centering
	\includegraphics[width=\linewidth]{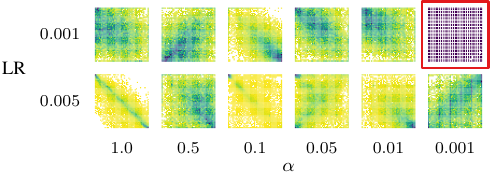}
	\caption{\textbf{Entangled disentanglement.} 
		The embedding spaces of DAEs  \cite{cha_orthogonality-enforced_2023} vary widely 
		when we change the learning rate LR and the batch-normalization hyperparameter~$\alpha$ of the model,  
		and the latent structure of the \texttt{XYC} dataset (shapes varying in 2D coordinates and color) is properly disentangled only with the right parameter choices (\framebox{DAEred}). 
		\ourmethod can \emph{topologically} assess the (hyper)parameter sensitivity of latent-space models.
	}\label{fig:dae}
\end{figure}

\textbf{Our Contributions.}\quad
Motivated by the need for respon\-sible latent-space models, and
encouraged by the promise of topological approaches to representational
variability~\citep{barannikov2021manifold,barannikov2022representation,zhou2021evaluating}, 
in this work, we use topology to map the multiverse of machine-learning models that rely on latent representations. 
In particular, we ask two guiding questions.
\begin{compactenum}[(Q1)]
	\item \textbf{\emph{Exploring} representational variability.}
	\emph{How do the latent representations of machine-learning models vary across different choices of model architectures, (hyper)parameters, and datasets?}
	\item \textbf{\emph{Exploiting} representational variability.}
\emph{How can~we use representational variability to efficiently train and select robust and reliable machine-learning models?}
\end{compactenum}
To address these questions, 
we make five contributions. 
\begin{compactenum}[(C1)]
	\item We introduce \ourmethod, a \textbf{topological multiverse framework} to describe and \emph{directly} compare both \emph{individual} latent spaces and \emph{collections} of latent spaces, 
	as summarized in \Cref{fig:pipeline}. 
	\item We capture \textbf{essential features of latent spaces} via persistence diagrams and landscapes, 
	allowing us to measure the pairwise \emph{(dis)similarity} of embeddings and statistically reason about their \emph{distributions}. 
	\item We prove \textbf{theoretical stability guarantees} for topological representations of latent spaces under projection. 
	\item We develop \textbf{scalable practical tools} to measure \emph{representational} (hyper)parameter sensitivity, 
	identify anomalous embeddings, 
	compress (hyper)parameter search spaces, 
	and accelerate model selection.
	\item We demonstrate the utility of our tools via \textbf{extensive
    experiments} in numerous latent-space multiverses. 
\end{compactenum}
Our work improves our understanding of representational variability in latent-space models, 
and it offers a \emph{structure-driven} alternative to existing performance-driven approaches in the responsible-machine-learning toolbox.

\begin{figure}[t]
	\centering
	\includegraphics[width=0.9\linewidth]{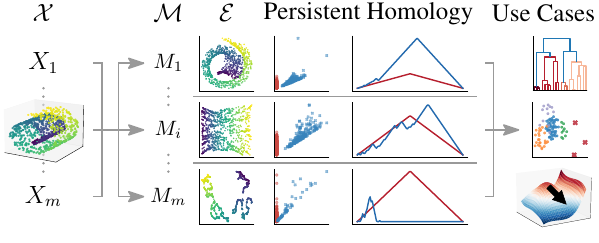}
	\caption{\textbf{The \ourmethod pipeline.} 
		For each model $\probe_i$ in our multiverse \multiverse, 
		\ourmethod computes the persistent homology associated with the embedding of a dataset $\dataset_i$ generated by~$\probe_i$, 
		yielding a set of embeddings \embeddings. 
		Thus enabled to compare the latent spaces of different models via the landscape distance of their persistence landscapes, 
		with \ourmethod, we can cluster, compress, detect outliers in, and analyze the sensitivity of (hyper)parameter configurations.
	}\label{fig:pipeline}
\end{figure}

\textbf{Structure.}\quad
Having given some background on persistent homology in \cref{Background}, 
we introduce \ourmethod, our multiverse framework for exploring and
exploiting representational variability in latent-space models in
\cref{Method}.
After discussing related work in \cref{Related-Work}, 
we gauge the practical utility of our framework through extensive experiments in \cref{Experiments},  
before concluding with a discussion in \cref{Conclusion}.
Extensive supplements are provided in \cref{apx:theory,apx:Background,apx:Related Work,apx:methods,apx:experiments}.
 \section{Background}
\label{Background}

This section briefly introduces \emph{persistent
homology}~(PH), the machinery for capturing essential features of data
that forms the basis of our framework.
Additional definitions and theorems can be found in \cref{apx:theory},
while \cref{apx:Background} contains background information on
\emph{latent-space models}.
Persistent homology~\citep{Barannikov94, Edelsbrunner10} is a framework used to analyze
topological characteristics of data at multiple spatial scales. 
It systematically quantifies the evolution of both
\emph{topological} and \emph{geometric} features, as tracked by
a \emph{filtration}, i.e., a consistent ordering of the elements of
a space.
Filtrations typically arise by approximating data with
\emph{simplicial complexes}, i.e., generalized graphs,
using metrics like the $L^2$ distance; we
will work with computationally efficient
\mbox{$\alpha$-complexes}~\citep{Edelsbrunner83a}.
Thus, given a topological space \topospace and a filtration
$\{\topospace_\filtrationparam\}_{\filtrationparam\in\reals}$, with each
$\topospace_\filtrationparam$ being a subcomplex of \topospace, persistent homology computes a sequence of homology groups $\{\Hom_\topodim(\topospace_\filtrationparam)\}_{\topodim\geq 0}$ for each~\filtrationparam. 
These groups capture \topodim-dimensional topological features such as
\emph{connected components}, \emph{cycles}, or \emph{voids} of
$\topospace_\filtrationparam$ at multiple resolutions, 
and they are \emph{invariant} to spatial transformations like translation,
rotation, and uniform scaling~(when working with normalized distances). 
Moreover, persistent homology varies continuously under continuous transformations of the space. 

Persistent-homology computations are typically summarized using \emph{persistence diagrams}, which provide a condensed representation for tracking topological features across multiple scales.
Formally, a persistence diagram $\dgm =\{(b_i, d_i)\}_{i}$ is a multiset of intervals, where $b_i$ represents the `birth time' and $d_i$ represents the `death time' of a given \topodim-dimensional topological feature, i.e., $b_i, d_i\in\reals \cup \{\infty\}$ with $b_i\leq d_i$.
As the space \dgms of persistence diagrams is cumbersome to work with
and does not afford efficiently-computable metrics, 
there are alternative representations of persistent homology, 
such as \emph{persistence landscapes}~(PL), 
which map persistence diagrams into a Banach space by transforming them
into piecewise-linear functions $\landscape_\topodim\colon
\dgms\to \reals$~\citep{bubenik2015statistical}.
This transformation allows us to  compare
topological descriptors using computationally efficient metrics,
and its calculation requires neither discretization nor additional
parameter choices.
For our work, the \emph{stability} of persistent homology and persistence landscapes under
perturbations and transformations of the data is particularly relevant: 
These descriptors are well-behaved under
structure-preserving embeddings~\citep{sheehy_persistent_2014,
	krishnamoorthy_normalized_2023}, 
and they capture \emph{both} geometric and topological properties of data~\citep{bubenik_persistent_2020}.
Thus, we select persistent homology and persistence landscapes as our lens for assessing representational variation.
 \section{Topological Multiverse Analysis}
\label{Method}

Having established the necessary background, 
we now introduce \ourmethod, our multiverse framework for exploring and exploiting representational variability in latent-space models via persistent homology. 
To define a multiverse of latent representations, 
we distinguish three categories of choices, 
namely,
\begin{inparaenum}[(1)]
	\item \emph{algorithmic choices} (i.e., model architecture and hyperparameters), 
	\item \emph{implementation choices} (such as optimizer, learning rate, number of epochs, and random seeds), and 
	\item \emph{data choices} (i.e., dataset and preprocessing).  
\end{inparaenum}
The mutually compatible options in each category give rise to three sets of valid choices, i.e., 
\emph{algorithmic choices} $\algo\in\algos$, 
\emph{implementation choices} $\implementation\in\implementations$, 
and \emph{data choices} $\data\in\datas$. 
We will also think of \algos, \implementations, and \datas as sets of \emph{vectors}, 
such that we can refer to $(\algo,\implementation,\data)\in \algos \times \implementations \times \datas$
by its parameter vector $\params \coloneq (\algo,\implementation,\data)’$ of cardinality $\nparams \coloneq \cardinality{\params}$.
Thus, we arrive at the notion of a \emph{latent-space multiverse}. 
\begin{definition}[Latent-Space Multiverse]\label{def:latent-space-multiverse}
	Given \emph{algorithmic choices}~\algos, \emph{implementation
  choices}~\implementations, and \emph{data choices}~\datas, 
	a \textbf{latent-space multiverse} \multiverse is a subset of $\algos
  \times \implementations \times \datas$. 	
	Each element $\params \in \multiverse$ is a
	\emph{universe} with an associated \emph{model}  $\model^\params \colon \reals^\anydimension \to \reals^\dimension$, 
	where \dimension is the desired embedding dimension (part of \params), 
	$\anydimension$ denotes a flexible input dimension, 
	and we drop ${}^\params$ for conciseness.
\end{definition}
In this work, we are interested in \emph{finite} latent-space multiverses, 
generated by \emph{discrete} subsets of \algos, \implementations, and \datas. 

\subsection{\ourmethod Pipeline}

Given a finite latent-space multiverse $\multiverse\subseteq\algos\times\implementations\times\datas$, 
to \emph{explore} representational variability, 
we would like to compare the topologies of individual latent spaces and statistically assess their distribution in \multiverse, 
which will also prove useful for \emph{exploiting} representational variability in practical applications. 
Unfortunately, working with the topologies of high-dimensional latent spaces directly, e.g., 
by comparing them via the~(normalized) bottleneck distance of persistence
diagrams derived from Vietoris--Rips complexes, 
is \emph{computationally prohibitive}. 
Furthermore, as persistence diagrams do not live in a \emph{Banach space}, 
we cannot reason about their distributions.
Hence, we instead propose the following \emph{scalable} framework for topological multiverse analysis, 
which we call \ourmethod (\textsc{Pr}ojected \textsc{e}mbedding \textsc{s}imilarity via \textsc{t}opological \textsc{o}verlays).\footnote{Although we advocate for the pipeline presented below, 
	its individual steps can be easily modified, 
	e.g., to use other topological descriptors. 
	See \cref{apx:pipeline-alternatives} for a detailed discussion. 
}  
For each model $\model\in\multiverse$ and a dataset \dataset (which can differ from the training data of \model), \ourmethod performs four steps.\footnote{While we keep the dataset \dataset fixed in our exposition for notational simplicity, 
	as we illustrate experimentally in \cref{Experiments}, 
	we can also use our multiverse approach to assess how changing \dataset affects our reconstruction of a (set of) latent space(s).
}
\begin{compactenum}[(S1)]
	\item \textbf{Embed data.}
	Compute the $\dimension$-dimensional \emph{embedding} ${\embedding\coloneq \model(\dataset)}$. \emph{Optionally,} \label{step:embedding}
	\begin{inparaenum}[(a)]
		\item approximate the \emph{diameter} of \embedding, and \label{step:approximate}
		\item \emph{normalize} \embedding by this approximation.\label{step:normalize}
	\end{inparaenum}
	\item \textbf{Project embeddings.} \label{step:projection}
	Project \embedding down to $\projdim\ll\dimension$ dimensions. 
	This can be done either \emph{deterministically}, e.g., via Principal Component Analysis (PCA), 
	or by generating a set of \emph{\nprojections random projections} \projections.\footnote{While the deterministic approach is particularly suited for comparing \emph{different} latent spaces, 
	leveraging randomness allows us to study variability within \emph{individual} latent spaces. 
}
	\item \textbf{Construct persistence diagrams.}
	Calculate the \emph{persistence diagram} $\dgm$ (or $\dgm_i$ for each
  projection $\embedding_i \in \projections$) based on its
  \topodim-dimensional \emph{$\alpha$-complex}. \label{step:persistence}
	\item \textbf{Compute persistence landscapes.}\label{step:landscapes}
	Vectorize the persistence diagram \dgm into a \emph{persistence landscape} $\landscape(\model)$ 
	(or $\dgm_i$ into $\landscape_i$ with  $\landscape(\model) \coloneq \nicefrac{\sum_{i\in[\nprojections]}\landscape_i}{\nprojections}$).
\end{compactenum}
Here, S\ref{step:approximate} and S\ref{step:normalize} replace exact
normalization, and S\ref{step:projection} and S\ref{step:persistence}
replace persistence-diagram computation based on Vietoris-Rips complexes
constructed in (potentially) high dimensions, all of which are
computationally costly.
S\ref{step:landscapes} guarantees that we operate in
a Banach space, and averaging here increases robustness if we work with
random projections. 
Since \ourmethod is based on persistent homology, 
it immediately benefits from PH's well-studied properties (see~\cref{Background}).
Finally, our pipeline is \emph{data-agnostic}, i.e., 
we can use any dataset \dataset to study variability in
\multiverse without requiring access to the training data---and we can even use \ourmethod to study
 variability in \emph{datasets} (see~\cref{Experiments}).

\subsection{\ourmethod Primitives}

By performing S\ref{step:embedding} to S\ref{step:landscapes} for each $\model\in\multiverse$, 
we obtain a set \landscapes of persistence landscapes, 
which allows us to achieve two fundamental tasks.
First, we can measure the distance between two latent spaces via the \emph{\ourmethod distance} (\ourdistance).\footnote{In the following, 
	to avoid ambiguity in our \ourmethod-related definitions, 
	we use a superscript $p$ to denote the choice of $L^p$ norm, 
	and a subscript $h$ to denote the maximum dimension of topological features considered. 
	Otherwise, we drop these superscripts and subscripts for simplicity when they are not decisive.
} 
\begin{definition}[\ourmethod Distance {[}\ourdistance{]}]
	\label{def:latent-space-landscape-distance}
	Given persistence landscapes $\landscape(\model_i)$ and $\landscape(\model_j)$, 
	the \textbf{$\ourmethod^\distancenorm_\topodim$ distance} up to topological dimension \topodim between $\model_i$ and $\model_j$ is 
	\begin{equation}\label{eq:distance}
		\text{\ourdistance}^\distancenorm_\topodim(\model_i,\model_j) 
		\coloneq 
		\sum_{x=0}^\topodim \distance_{\landscape^\distancenorm}\! \left(\landscape_x(\model_i),\landscape_x(\model_j)\right)
		\;,
	\end{equation}
	where $\distance_{\landscape^\distancenorm}$ is the landscape distance based on the $L^p$ norm.
\end{definition}
Second, we can assess the variance in a set of latent spaces via the \emph{\ourmethod variance} (\ourvariance).
\begin{definition}[\ourmethod Variance {[}\ourvariance{]}]
	\label{def:latent-space-landscape-variance}
	Given a set \landscapes of persistence landscapes of cardinality ${\nlandscapes \coloneq \cardinality{\landscapes}}$, 
	the \textbf{$\ourmethod^\distancenorm_\topodim$ variance} up to topological dimension \topodim of~\landscapes~is
	\begin{equation}\label{eq:sensitivity}
		\text{\ourvariance}^\distancenorm_\topodim(\landscapes) \coloneq \frac{1}{\nlandscapes}
		\sum_{x=0}^{\topodim}
		\sum_{\landscape\in\landscapes^x}
		\left(\lVert \landscape \rVert_\distancenorm - \mu_{\landscapes^x}\right)^2\;,
	\end{equation}
	where $\lVert\landscape\rVert_\distancenorm$ is the $L^\distancenorm$-based landscape norm, 
	$\landscapes^x$ denotes the landscape parts associated with $x$-dimensional topological features, 
	and $\mu_{\landscapes^x}$ is the mean of landscape norms in $\landscapes^x$.
\end{definition}
Although the modifications to an exact pipeline made by our \ourmethod framework induce some changes in our representations, 
they retain the essential features of our original latent spaces. 
As a result, the error we introduce into our measurements in  \cref{eq:distance,eq:sensitivity} is bounded both theoretically and empirically, 
as we show in \cref{Theory,Experiments}.

\subsection{\ourmethod Applications}

As we demonstrate in our experiments (\cref{Experiments}), 
our multiverse framework and \ourmethod primitives are useful for exploring and exploiting the representational variability of latent-space models 
in several different settings. 
In particular, they can help us 
\begin{inparaenum}[(1)]
	\item evaluate (hyper)parameter sensitivity,
	\item detect anomalous embeddings, and
	\item cluster and compress (hyper)parameter search spaces.
\end{inparaenum}

\textbf{Sensitivity Analysis.}\quad
Following the reasoning sketched in \cref{Introduction}, 
choices amplifying representational variability are both analytically interesting and practically problematic. 
Thus motivated to study (hy\-per)pa\-ram\-e\-ter sensitivity in latent-space multiverses~\multiverse,  
our goal here is to quantify the structural variation in the embedding space when introducing controlled variation in $\params\in\multiverse$. 
We seek to assess the \emph{local} and \emph{global} sensitivity in \multiverse, 
as well as the sensitivity at \emph{individual} coordinates $\params$ in $\multiverse$.
To this end, we introduce our \emph{\ourmethod sensitivity} scores (\oursensitivity). 
\begin{definition}[\ourmethod Sensitivity {[}\oursensitivity{]}]
	\label{def:latent-space-parameter-sensitivity}
	Given a multiverse \multiverse, fix a model dimension $i$,
	and define an equivalence relation $\sim_i$ such that $\params' \sim_i \params'' \Leftrightarrow \params'_j = \params''_j$ for all $\params', \params''\in\multiverse$ and $j \neq i$, 
	yielding $\neqclasses_i$ equivalence classes~$\eqclasses_i$.
	
	The \textbf{\emph{individual} $\ourmethod^\distancenorm_\topodim$ sensitivity}
	of equivalence class $\eqclass\in\eqclasses_i$ in \multiverse is
\begin{equation}
		\text{\ourcoordinatesensitivity}_\topodim^\distancenorm(\eqclass\mid\multiverse) \coloneq \sqrt{\text{\ourvariance}_\topodim^\distancenorm(\landscapes[\eqclass])}\;,
	\end{equation}
where $\landscapes[\eqclass]\subset\landscapes$ is the set of
	landscapes associated with models in equivalence class \eqclass. 
Aggregating over all equivalence classes in $\eqclasses_i$, we obtain
	the \textbf{\emph{local} $\ourmethod^\distancenorm_\topodim$ sensitivity} of
	\multiverse in model dimension $i$ as 
\begin{equation}
		\text{\oursensitivity}_\topodim^\distancenorm(\multiverse \mid i) \coloneq \sqrt{\frac{1}{\neqclasses_i}\sum_{\eqclass\in\eqclasses_i}\text{\ourvariance}_\topodim^\distancenorm(\landscapes[\eqclass])}\;.
	\end{equation} 
Finally, aggregating over all $\nparams = \cardinality{\params}$
	dimensions of models in \multiverse yields the \textbf{\emph{global}
		$\ourmethod^\distancenorm_\topodim$ sensitivity} of \multiverse, i.e., 
\begin{equation}
		\text{\oursensitivity}_\topodim^\distancenorm(\multiverse) \coloneq \sqrt{\frac{1}{\nparams}\sum_{i\in[\nparams]}\frac{1}{\neqclasses_i}\sum_{\eqclass\in\eqclasses_i}\text{\ourvariance}_\topodim^\distancenorm(\landscapes[\eqclass])}\;.
	\end{equation}
\end{definition}
Note that when \multiverse varies only in one dimension, 
the individual, local, and global \ourmethod sensitivities are identical, such that we can simply speak of the \emph{\ourmethod sensitivity}.

\textbf{Outlier Detection.}\quad
Since models with anomalous latent spaces are, by definition, not robust, 
we should understand for which sets of choices they occur and avoid working with them in practice. 
To detect such anomalous latent spaces in a multiverse \multiverse with associated landscapes \landscapes, 
we can exploit their Banach-space structure, 
which allows us to use the \ourvariance~(\cref{def:latent-space-landscape-variance}), 
along with standard statistical approaches, 
to identify landscapes with anomalous norms.

\textbf{Clustering and Compression.}\quad
To identify interesting structure in a collection of latent spaces (arising, e.g., from a grid search), 
we can \emph{cluster} the collection, 
represented by a multiverse \multiverse with associated landscapes \landscapes, 
based on the \ourmethod distance (\cref{def:latent-space-landscape-distance}), 
using \emph{any clustering method} based on pairwise distances. 
Reducing the costs of exhaustive (hyper)parameter searches, 
however, requires us to lower the number of configurations considered in detail. 
As we demonstrate experimentally, 
if two latent spaces are topologically close in our \emph{target} setting 
(i.e., a search space we would like to \emph{avoid exploring exhaustively}), 
they may also be close in a \emph{proxy} setting (i.e., a search space we \emph{will explore (or have explored) exhaustively}), 
such as when training on a related dataset. 
This permits us to reuse knowledge generated from proxy settings
for our target setting, motivating the task of search-space compression.
Given results from a proxy setting \proxy, 
to \emph{compress} the search space in our target setting \multiverse, 
we define a threshold~$\epsilon$ and select \emph{representatives} $\representatives\subseteq\multiverse$ such that for each $\model_i\in\multiverse$, 
there exists a representative $\model_j\in\representatives$ 
with ${\text{\ourdistance}_\proxy(\model_i,\model_j) \leq \epsilon}$, 
where $\text{\ourdistance}_\proxy$ denotes the \ourmethod distance in proxy setting \proxy. 
\cref{apx:compression} provides more details on how to pick
suitable representatives in practice. 

\subsection{\ourmethod Stability}\label{Theory}

To ensure scalability, \ourmethod computes topological descriptors on
low-dimensional projections of embeddings, rather than working
in a high-dimensional latent space.
Therefore, we would like to ascertain that the distortion introduced by these projections
remains bounded.
To achieve this, we require the notion of a \emph{multiverse metric
space}~(MMS).
We will work with the $L^2$ landscape norm
and consider topological features up to dimension $2$~(i.e., $\distancenorm
= \topodim = 2$), dropping the superscript \distancenorm and the
subscript \topodim for notational conciseness, 
and deferring all proofs to \cref{apx:theory}.
\begin{definition}[Multiverse Metric Space \MMS\ {[}MMS{]}]
	\label{def:mms}
	For a multiverse \multiverse with associated embeddings \embeddings, 
	we define the topological distance of embeddings in \embeddings as   
	\begin{equation}\label{eq:topological-distance}
		\distance_\topological(\embedding_i,\embedding_j) \coloneq \distance(\dgm(\embedding_i),\dgm(\embedding_j)) \;,
	\end{equation} 
	where $\distance$ can be any distance between persistence
  representations~(e.g., diagrams or landscapes).
A \textbf{multiverse metric space} is the tuple $(\embeddings,\distance_\topological)\eqcolon \MMS$.
\end{definition}
When working with \projdim-dimensional embedding \emph{projections}, 
we operate in a \emph{projected multiverse metric space} (PMMS). 
\begin{definition}[Projected Multiverse Metric Space $\MMS^\projdim$ {[}PMMS{]}]\label{def:pmms}
	Given an MMS $\MMS = (\embeddings,\distance_\topological)$, 
	fix the projection dimension ${\projdim \in \naturals}$,
	s.t.\ $\projdim \leq \dimension_i$ for $\embedding_i\in \embeddings$. 
For a projector $\projector\colon \reals^\anydimension \to \reals^\projdim$, 
	let
  $\projections \coloneq \{\projector(\embedding) \mid \embedding\in \embeddings\}$
  be the set of \emph{projected} embeddings.
Using $\distance_\topological$ from \cref{eq:topological-distance}, 
	a \textbf{projected multiverse metric space} is defined as 
	$\MMS^\projdim\coloneq (\projections,\distance_\topological)$. 
\end{definition}
Relating an MMS \MMS to its \projdim-dimensional counterpart
$\MMS^\projdim$, we arrive at the notion of \emph{topological loss},
i.e., the decrease in topological fidelity due to our projection.
\begin{definition}[Topological Loss]
	\label{def:topological-loss}
	Given an MMS \MMS, 
	a projector $\projector\colon \reals^\anydimension \to \reals^\projdim$,  
	and an associated PMMS $\MMS^\projdim$,  
	with
	$\projection_i \coloneq \projector(\embedding_i)$, 
	the \textbf{topological loss} \topologybound of $\MMS^\projdim$ 
	is the maximum distance between elements of $\embeddings$
	and $\projections$ measured by $\distance_\topological$, i.e.,
	\begin{equation} \label{eq:topological-loss} 
		\topologybound \coloneq \max_{\embedding_i\in \embeddings} \distance_\topological(\embedding_i,\projection_i)\;.
	\end{equation}
\end{definition}
This topological loss bounds the pairwise-distance perturbation of our metric space under projection.
\begin{restatable}[Metric-Space Preservation under Projection]{theorem}{thmMMSPreservation}
	\label{thm:ClusteringPreservation}
	Given an MMS \MMS and an associated PMMS $\MMS^\projdim$ with
  topological loss \topologybound, we can bound the pairwise-distance perturbation
	under projection as
	$\MMS^\projdim[i,j] \leq \MMS[i,j] + 2\topologybound\;.$
\end{restatable}
Consequently, as the topological loss \emph{increases}, 
our precision in distinguishing embeddings \emph{decreases} in a controlled manner, 
and we can further bound the \ourmethod variance under projection as follows.
\begin{restatable}[\ourmethod Variance under Projection]{theorem}{thmSensitivityPreservation}
  \label{thm:SensitivityPreservation}
  Consider an MMS $\MMS = (\embeddings,\distance_\topological)$ 
  with the landscape distance $\distance_\topological(\embedding_i,\embedding_j) \coloneq 
  \distance(\landscape(\embedding_i),\landscape(\embedding_j))$ 
  and associated persistence landscapes $\landscapes_\MMS$. 
  Further, let $\MMS^\projdim$ be a PMMS 
  with a topological loss $\topologybound$.
  Then we can bound the maximal change in \emph{any} persistence
  landscape norm as
\begin{equation}
     \|\landscape(\embedding_i)\|  - \topologybound \leq  \|\landscape(\projection_i)\| \leq  \|\landscape(\embedding_i)\| + \topologybound \;.
     \label{eq:Landscape norm change}
  \end{equation}
Given the \ourmethod variance of \MMS,
$\text{\ourvariance}(\landscapes_\MMS)$, we can bound the \ourmethod
variance of $\MMS^\projdim$, i.e., $\text{\ourvariance}(\landscapes_{\MMS^\projdim})$, as 
\begin{equation}
  \cardinality{\text{\ourvariance}(\landscapes_\MMS)-\text{\ourvariance}(\landscapes_{\MMS^\projdim})} 
  \leq \frac{4\topologybound}{\nlandscapes} \sigma_i + \frac{\left(2\topologybound\right)^2}{\nlandscapes}\;,
\end{equation}
where $\sigma_i \coloneq \sum_{i=1}^{\nlandscapes} (\lVert \landscape(\embedding_i)\rVert - \mu_\embeddings)$.
\end{restatable}

As a result, \ourmethod is stable as long as we can control the
approximation error induced by the choice of projector function.
To this end, \cref{apx:theory} provides explicit bounds for several
common classes of projection functions.

\subsection{\ourmethod Complexity}
\label{methods:complexity}

While scalability is a common concern in computational topology,  
our framework is specifically designed for scalability.  
We present a detailed complexity analysis in \cref{apx:computational-complexity:theoretical}, 
showing that, overall, \ourmethod's computations are approximately linear in the number of samples in \dataset. 
In
\cref{apx:computational-complexity:empirical,apx:computational-complexity:comparative},
we also validate \ourmethod's time complexity empirically, demonstrating
that \ourmethod distances can be computed faster than related (dis)similarity measures.
 \section{Related Work} \label[section]{Related-Work}

\ourmethod connects two strands of literature, i.e., 
\emph{topological approaches} and \emph{multiverse approaches} in machine learning. 
Since our framework additionally draws on numerous other fields, 
we provide an extended discussion in \cref{apx:Related Work}. 

\textbf{Topological Approaches in Machine Learning.}\quad
Topological approaches have been used to analyze and control representational variation,
  leading to regularization terms that preserve  topological characteristics~\citep{trofimov2023learning, Moor20a, Waibel22a},
scores for assessing disentanglement~\citep{zhou2021evaluating}
or quality~\citep{Rieck15b, Rieck16a},
as well as methods for learning disentangled
representations~\citep{balabin2023disentanglement}, 
studying neural networks~\citep{klabunde2023similarity,
kostenok2023uncertainty, purvine2023experimental, Rieck19a},
measuring generative quality~\citep{kim2023robust}, 
and enabling zero-shot training \citep{moschella2023relative}.
Drawing on topological concepts to analyze differences between latent spaces,
the \emph{manifold topology
divergence} and the
\emph{representation topology
divergence}~(RTD)~\citep{barannikov2021manifold,barannikov2022representation} 
are closest to our work. 
However, these methods focus on \emph{pairwise}
comparisons of \emph{aligned} data, 
exhibit unfavorable scaling behavior, 
and do not enjoy theoretical fidelity guarantees. 
By contrast, 
as the first method for studying representational variability that leverages persistence landscapes, 
\mbox{\ourmethod} enables a \emph{multiverse} analysis in
terms of models, (hyper)parameters, and datasets, 
readily handles \emph{unaligned}
embeddings, and quantifies all results in terms of \emph{distance metrics},
thus improving their interpretability. 

\textbf{Multiverse Approaches in Machine Learning.}\quad
Multiverse analysis \citep{Steegen16a} aims to reduce arbitrariness and increase transparency in data
analysis via the joint consideration of multiple reasonable analytical scenarios.
Precursors of a multiverse perspective in machine
learning have assessed
hyperparameter choices~\citep{kumar2020implicit, sivaprasad2020optimizer, smith1803disciplined,
zhang2019algorithmic}, studied the causal structure of latent
representations~\citep{leeb_exploring_2022}, or compared different
models~\citep{vittadello_model_2021, diedrichsen_comparing_2020,
diedrichsen_representational_2017}.
By contrast,  
\ourmethod provides a unified framework for the structural analysis of latent spaces across models, (hyper)parameters and datasets.
While \citet{bell2022modeling} consider a model multiverse, 
their work differs in its goals and its methods. 
In particular, they focus on the
\emph{performance-driven exploration} of \emph{continuous} search spaces
via tools from \emph{statistics}, while we pursue the
\emph{structure-driven characterization} of \emph{discrete} search
spaces using tools from \emph{topology}.
 \section{Experiments} \label[section]{Experiments}

In our experiments, 
we ask three questions:
\begin{compactenum}
	\item[(Q0)] \textbf{\ourmethod's distinctive properties.}\\ \emph{How does \ourmethod relate to existing measures of representational (dis)similarity and variability?}
	\item[(Q1)] \textbf{\ourmethod for \emph{exploring} representational variability.}\\
	\emph{How can \ourmethod help us understand representational variability across different choices of model architectures, (hyper)parameters, and datasets?}
	\item[(Q2)] \textbf{\ourmethod for \emph{exploiting} representational variability.}\\
	\emph{How can \ourmethod help us to efficiently train and select  robust and reliable machine-learning models?}
\end{compactenum}

To address our guiding questions, 
we generate multiverses for two types of generative models, i.e., 
\emph{variational autoencoders} as generators of \emph{images}, 
and \emph{transformers} as generators of \emph{natural language}. 
We are particularly interested in the impact of 
\emph{algorithmic} choices \algos, \emph{implementation} choices \implementations, 
and \emph{data} choices \datas on the generated representations. 
Further experiments (including on dimensionality reduction),  
a multiverse analysis of the choices involved in the \ourmethod pipeline, 
and more details on all results reported here  
can be found in \cref{apx:experiments}.\footnote{Code:  \href{https://github.com/aidos-lab/Presto}{https://github.com/aidos-lab/Presto}. 
	Reproducibility package: \oururl.
	Unless otherwise noted, all experiments use normalization. 
} 

\textbf{VAE Multiverses.}\quad
In brief, our VAE experiments study representational variation in the following dimensions. 
\begin{compactenum}
	\item[\algos.] 
	We consider three VAE architectures: 
	\begin{inparaenum}[(1)]
		\item \bvae \citep{higgins2017betavae}, 
		\item \info \citep{zhao2018infovae}, and
		\item \wae \citep{tolstikhin2019wasserstein}.
	\end{inparaenum}
	For each architecture, 
	we investigate the interplay between hyperparameter 
	choices and latent-space structure when navigating trade-offs between reconstruction bias and KL-divergence weight in conjunction with loss variations and kernel parameters.
	\item[\implementations.]
	Specifically for \bvae, in \cref{apx:vae-implementation-multiverse}, 
	we explore the relation between $\beta$ and five implementation choices:
	\begin{inparaenum}[(1)]
		\item batch size $b$,
		\item hidden dimensions $h$, 
		\item learning rate $l$,
		\item sample size $s$, 
		\item and train-test split $t$.
	\end{inparaenum}  
	\item[\datas.] 
	 We train on five datasets: 
	 \begin{inparaenum}[(1)]
	 	\item \celeba,
	 	\item \cifar,
	 	\item \dsprites,
	 	\item \fashionmnist, and
	 	\item \mnist.
	 \end{inparaenum}
\end{compactenum}
For a detailed description, see \cref{apx:experiment-hyperparameter-multiverse,apx:vae-implementation-multiverse}. 

\textbf{Transformer Multiverses.}\quad
Based on access to pretrained language models only,
we focus our transformer experiments on algorithmic and data choices.
In \cref{apx:transformers:norms,apx:transformers:distances,apx:transformers:sensitivities}, 
we additionally use transformers to study the impact of implementation choices in the
\ourmethod pipeline. 
\begin{compactenum}
	\item[\algos.] 
	We consider six transformer models: 
	\begin{inparaenum}[(1)]
		\item \ada,
		\item \mistral,
		\item \distilroberta,
		\item \MiniLM,
		\item \mpnet, and
		\item \qadistilbert.
	\end{inparaenum}
	The first two models are large language models from OpenAI and MistralAI, respectively, 
	whereas the other four models are sentence transformers taken from the \emph{sentence-transformers} library \cite{reimers2019sentencebert}.
	\item[\implementations.]
	To analyze the implementation multiverse of \ourmethod pipeline choices, 
	we introduce variation in the 
	\begin{inparaenum}[(1)]
		\item number of samples \nsamples embedded, 
		\item number of projection components \projdim considered, as well as in  
		\item embedding-projection method (PCA vs. random projections) and number of random projections \nprojections.
	\end{inparaenum}  
	\item[\datas.] 
	We probe each trained model by embedding abstracts from four summarization datasets, i.e., 
	\begin{inparaenum}[(1)]
		\item \arxiv,
		\item \bbc,
		\item \cnn, and
		\item \patents,
	\end{inparaenum}
	all of which are available via \emph{HuggingFace}.
\end{compactenum}
See \cref{apx:presto-transformers-description} for a detailed description. 

\begin{figure}[t]
	\centering
	\includegraphics[width=0.975\linewidth]{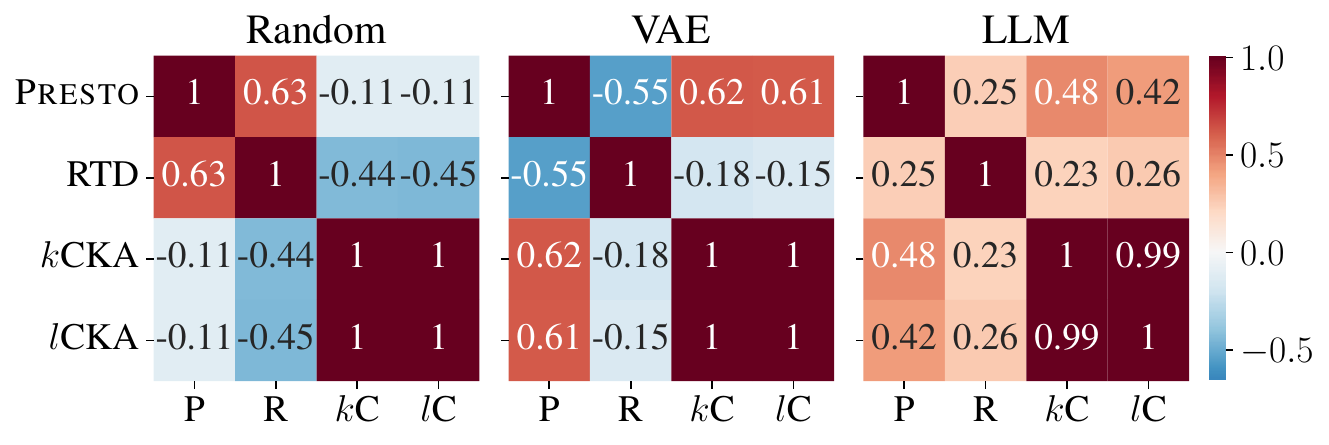}
	\caption{\textbf{Comparing \ourmethod distances with other measures.}
		We show the Pearson correlations between \ourmethod, 
		RTD, $k$CKA, and $l$CKA, 
		on random data (left), VAE embeddings (center), and LLM embeddings (right).
\mbox{\ourmethod} captures representational variation differently from existing methods.
}\label{fig:comparisons-competitors}
\end{figure}

\subsection{\ourmethod's Distinctive Properties}
\label{exp:properties}
With our experimental setup in place, 
we turn to our zeroth guiding question: 
understanding \ourmethod's distinctive properties. 
To begin, 
we compare \ourmethod distances with other measures of representational (dis)similarity 
on the basic task of \emph{pairwise} comparisons between \emph{aligned} embeddings 
(a limitation imposed by competitor methods). 
Summarizing the correlation between \ourmethod and RTD, 
both topology-based \emph{dissimilarity} measures, 
as well as RBF-kernel and linear Centered Kernel Alignment ($k$CKA and $l$CKA), 
both \emph{similarity} measures, 
in \Cref{fig:comparisons-competitors}, 
we observe that there is no consistent relationship between \ourmethod distances, RTDs, and CKA scores. 
This indicates that \ourmethod captures variation in latent-space structure differently from existing methods,  
which appears desirable, 
given the known limitations of existing approaches \citep[cf.][]{davari2023reliability}.

To understand what exactly is captured by \ourmethod, 
in the left panel of \Cref{fig:geometric-generative}, 
(see also \cref{apx:exp:geomgen}), 
we correlate \emph{unnormalized} \ourmethod distance matrices for 
different VAE hyperparameter multiverses with estimates of \emph{geometric} distances
between pairs of latent spaces (i.e., Pearson distances between random, aligned metric subspaces),
displaying the distribution of correlations over multiple random draws. 
We expect to see \emph{some} correlation between \ourmethod distances and geometric distances 
because our framework is based on PH, 
which also captures some geometric properties \cite{bubenik_persistent_2020}.
In line with this expectation, 
we see that \ourmethod distances are correlated with \emph{geometric} distances, 
albeit to different extents across models and datasets.
In the right panel of \Cref{fig:geometric-generative}, 
we further examine the relationship between landscape norms, 
the foundation of \ourmethod \emph{variances}, 
and the performance of models in the \bvae hyperparameter multiverse. 
We observe that while larger landscape norms are, overall, associated with larger losses, 
models with similar performance exhibit substantial variability in their landscape norms. 
This suggests that the variability captured by \mbox{\ourmethod} is orthogonal to variability in performance
(an impression further corroborated by an extended experiment in \cref{apx:Performance}),
underscoring \ourmethod's capacity to complement performance-based metrics,  
shed light on variability in \emph{similarly-performing} models, 
and promote \emph{representational stability} as a target in model evaluation.

\begin{figure}[t]
	\centering
	\includegraphics[width=\linewidth]{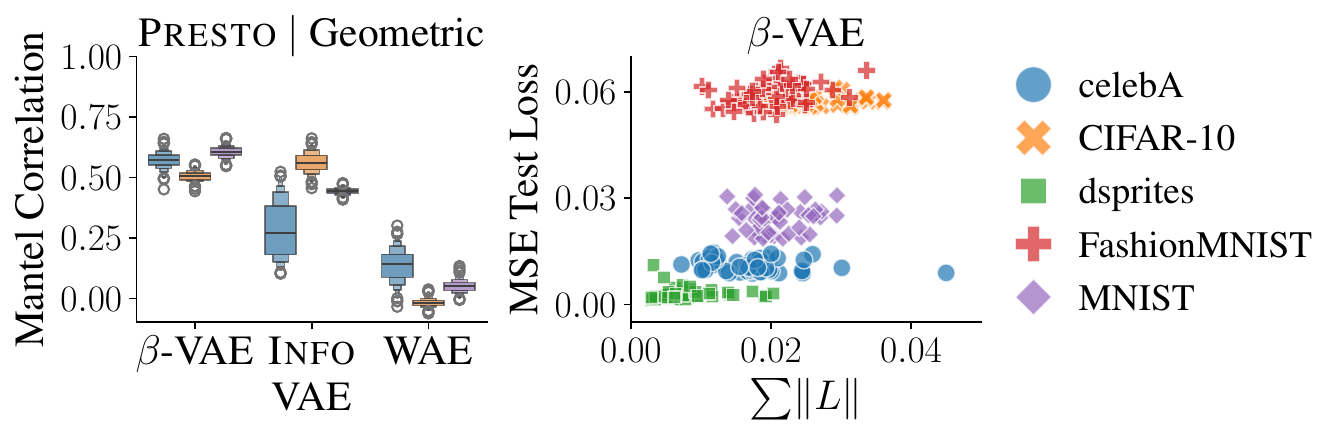}
	\caption{\textbf{Comparing \ourmethod, latent-space geometry, and model performance.}
		We show the distribution of correlations between \ourmethod distances 
		and geometric latent-space distances in the VAE multiverse, 
		estimating geometric distances based on the Pearson distance between Euclidean metric spaces of aligned random samples of size $512$ 
		over $256$ random draws (left),  
		as well as the relationship between landscape norms and model performance for \bvae (right). 
\mbox{\ourmethod} captures geometric similarity between latent spaces and is orthogonal to  performance.
}\label{fig:geometric-generative}
\end{figure}

\begin{figure}[t]
	\centering
	\includegraphics[width=\linewidth]{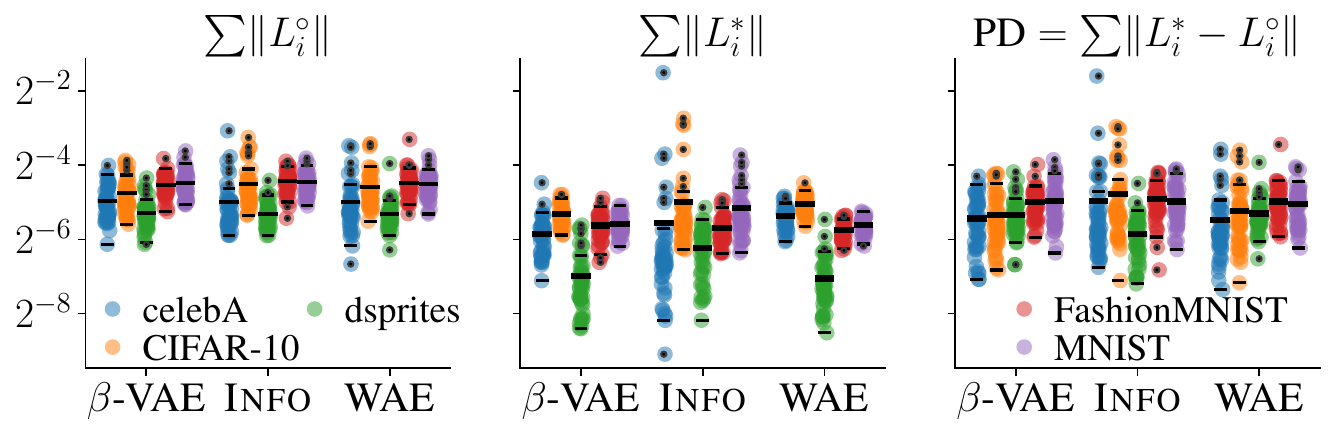}
	\caption{\textbf{Landscape-norm distributions in our VAE hyperparameter multiverse.} 
		We show the distribution of landscape norms after initialization (left) and training (center), 
		as well as the distribution of \ourmethod distances between the landscape at initialization and the landscape after training (right). 
		Thick lines indicate means, 
		thin black lines indicate interquartile range, 
		and black dots indicate outliers. 
		Training differentially affects landscape norms across models and datasets.
	} \label{fig:vae-landscape-norms}
\end{figure}

\begin{figure}[t]
	\centering
	\includegraphics[width=\linewidth]{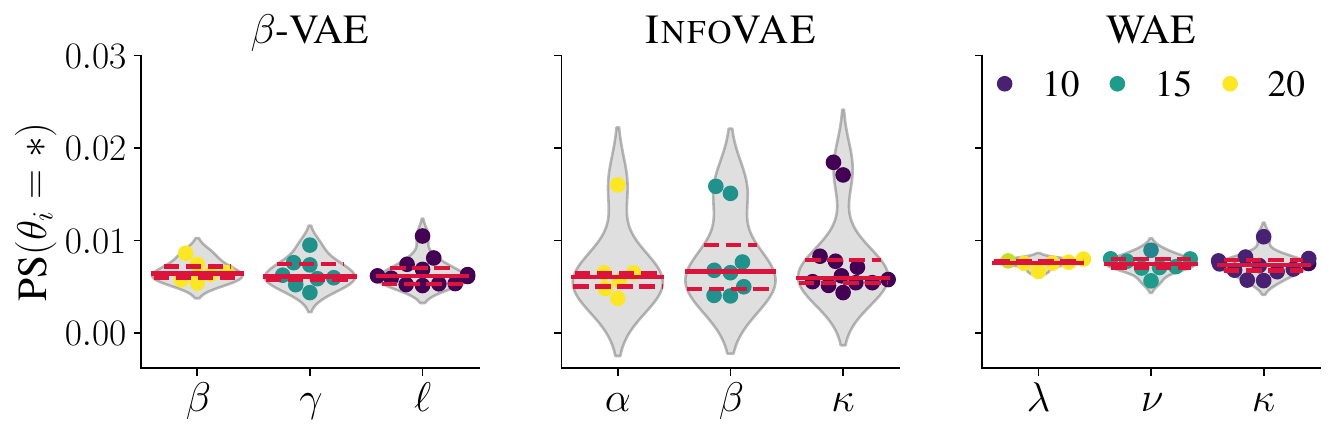}
	\caption{\textbf{Distributions of individual \ourmethod sensitivity in our VAE hyperparameter multiverse.} 
		We show the distribution of individual \ourmethod sensitivity scores for all equivalence classes of models that vary a particular parameter $\params_i$ while keeping the others constant. 
		Marker colors indicate the number of observations per equivalence class, 
		the solid red line marks the median, and dashed red lines indicate interquartile range. 
\mbox{\info} exhibits the largest variability in hyperparameter sensitivities.
} \label{fig:vae-sensitivity}
\end{figure}

\subsection{Exploring Representational Variability}
\label{exp:explore}
Reassured by \ourmethod's distinctive properties, 
we now leverage our framework to explore representational variability in our VAE and transformer multiverses. 
We find that landscape norms are approximately normally distributed, 
such that they permit standard statistical approaches to outlier detection. 
As shown in \Cref{fig:vae-landscape-norms}, 
training affects landscape norms differentially, 
depending on model and dataset choices. 
We observe that \info exhibits a larger fraction of anomalous configurations than \bvae and \wae, 
which motivates us to explore individual \ourmethod sensitivities for the main hyperparameters of our VAE models. 
Studying the distribution these sensitivities, 
depicted in \Cref{fig:vae-sensitivity}, 
reveals that \info has the largest variability in hyperparameter sensitivities---i.e., for \info, 
the effect of changing a hyperparameter depends more strongly on the position in the hyperparameter space than for \bvae and \wae.
We conclude that \mbox{\info} is less representationally stable than its contenders, 
which is confirmed by an analysis of the local and global \ourmethod sensitivities of all models in \cref{apx:vae-sensitivities}. 

Turning to our transformer multiverses, 
and further demonstrating the exploratory power of \ourmethod, 
in \Cref{fig:transformer-mms-compare}, 
we use our topological tools to directly compare the multiverse metric spaces (MMSs) of our transformer models. 
We see that we can distinguish large language models from smaller language models based on topological comparisons between their embedding multiverses. 
Furthermore, topological comparisons between multiverse metric spaces have higher discriminatory power than comparisons based on Mantel correlation \citep{mantel1967detection}. 

\begin{figure}[t]
	\centering
	\includegraphics[height=0.325\linewidth]{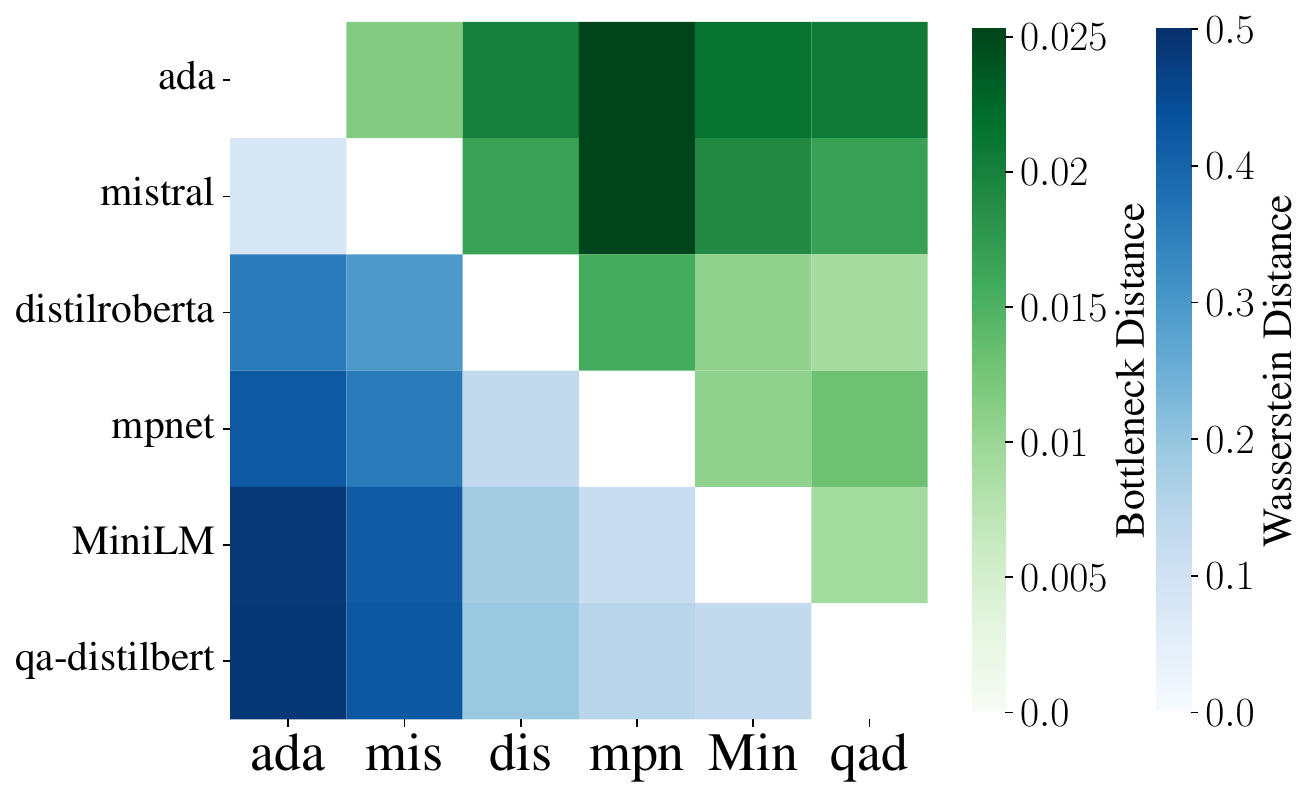}~\includegraphics[height=0.325\linewidth]{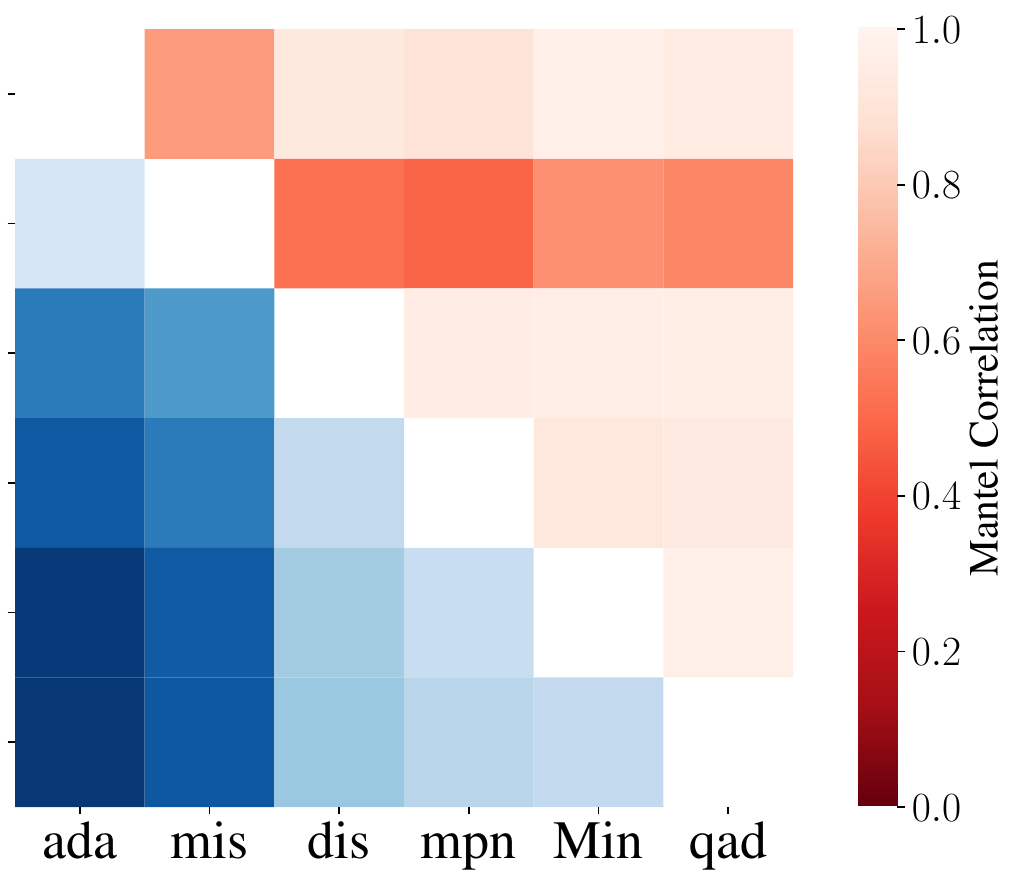}
	\caption{\textbf{Comparing MMSs across transformer models.}
		We show the (dis)similarity between the transformer multiverses associated with each of our models, 
		as measured by the Wasserstein distance (lower triangles), the bottleneck distance (upper triangle left), 
		or the Mantel correlation (upper triangle right). 
Topological distances provide a more nuanced perspective than permutation-based correlation assessments, 
			and they clearly distinguish the MMSs of large language models from those of smaller models.
} \label{fig:transformer-mms-compare}
\end{figure}
\begin{figure}[t]
	\centering
	\includegraphics[width=0.9\linewidth]{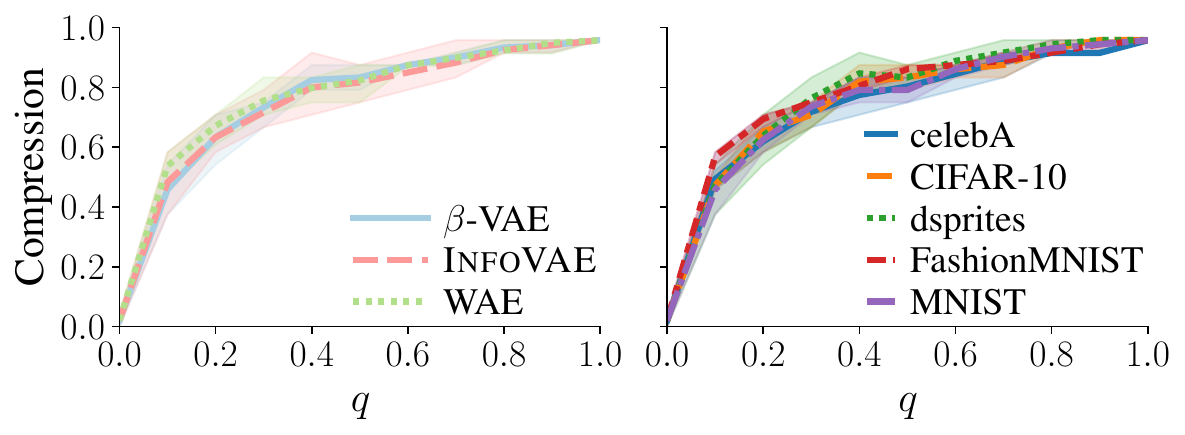}
	\caption{\textbf{Compressing the VAE hyperparameter multiverse.} 
		We show the compression of the hyperparameter multiverse split by models (left) and datasets (right) achieved by restricting the hyperparameter search to set-cover representatives which guarantee that each universe has a representative at distance no larger than the $q$th quantile of the distribution of pairwise distances in the multiverse. 
With \ourmethod, we can halve the size of the hyperparameter search space while ensuring low topological distortion.
}\label{fig:vae-compression}
\end{figure}

\subsection{Exploiting Representational Variability}
\label{exp:exploit}
In our explorations of representational variability, 
we have seen that \ourmethod supports sensitivity analysis and outlier detection for latent-space models.  
Encouraged by these findings, 
we further investigate how our framework can leverage representational variability (and the lack thereof) 
to facilitate the efficient selection of representationally robust and reliable latent-space models in practice. 
As shown in \Cref{fig:vae-compression}, 
with \ourmethod, we can compress a VAE hyperparameter search space by $50\%$, 
ensuring that each coordinate is matched to a structurally similar representative in the compressed space. 
This suggests that there exist opportunities for environmentally conscious, 
yet empirically sound hyperparameter selection based on topological insights into latent-space structure 
(see \cref{apx:exp:low-complexity-training} for an additional experiment exploring such opportunities). 

In \Cref{fig:reusing-hyperparameters}, 
we perform MMS comparisons with \ourmethod 
to assess when the widespread custom of reusing hyperparameter knowledge across datasets is topologically justified. 
We find that among our VAE models, 
only \mbox{\bvae} exhibits the cross-dataset latent-space consistency required for such a transfer. 
Theoretically sound knowledge transferability promotes sustainable yet rigorous machine-learning practices, 
and \ourmethod appears as a promising tool to evaluate the cross-dataset representational consistency of latent-space models
that is required to ensure it. 

Finally, \cref{apx:DimRed} shows how \ourmethod can be  leveraged in the context of data analysis, e.g., 
when reasoning about the results of \emph{non-linear dimensionality-reduction} methods.

\begin{figure}[t]
	\centering
	\includegraphics[width=\linewidth]{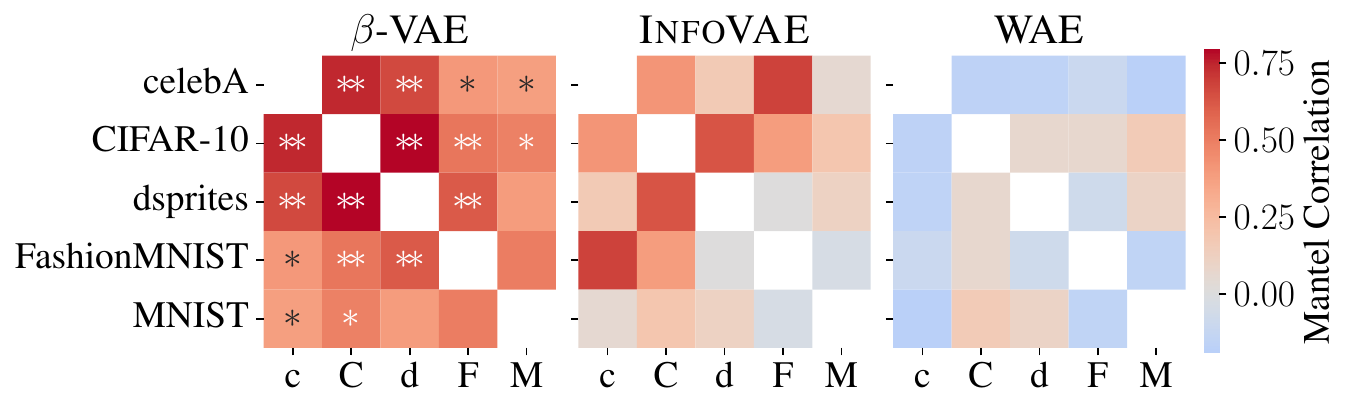}
	\caption{\textbf{Reusing hyperparameter knowledge.} 
		We show the Mantel correlation (color) between the hyperparameter multiverses of our three VAE models when trained on five different datasets, 
		annotating Bonferroni-corrected $p$-values at the $95\%$ ($\ast$) and the $99\%$ ($\ast\ast$) significance level. 
The topological reusability of hyperparameter knowledge strongly depends on the chosen model.
} \label{fig:reusing-hyperparameters}
\end{figure}
 \section{Discussion and Conclusion}
\label{Conclusion}
We introduced \ourmethod, 
a \emph{topological multiverse framework} to describe and relate (collections of) latent representations.
\ourmethod flexibly and scalably compares spaces varying in cardinality and dimension, 
surpassing existing work in generality 
while still capturing salient topological signal
and benefiting from theoretical stability guarantees.
By offering novel topological diagnostics for \emph{distributions} of latent spaces, 
\ourmethod unlocks a \emph{structure-driven} alternative for studying representational variability in generative models, 
complementing performance-driven approaches.
Drawing on the notion of \emph{multiverse metric spaces}, 
we used \ourmethod to develop scalable practical tools 
for efficiently evaluating and selecting latent-space models, 
including VAEs and transformers, 
across a wide range of configurations. 

\paragraph{Limitations.}
\ourmethod allows us to measure the sensitivity of latent-space models to changes in algorithmic, implementation, and data choices, 
and we can identify outliers among latent spaces. 
To the best of our knowledge, 
no suitable baselines currently exist for these purposes, 
since existing methods are based on model performance.
Hence, exploring alternative approaches to sensitivity analysis and
outlier detection for latent-space models based on their internal
representations constitutes a crucial avenue for future work.

Moreover, each step in our \ourmethod framework offers a number of choices. 
For example, we can opt to work with \emph{normalized} or \emph{unnormalized} embeddings, 
and choose \emph{deterministic} or \emph{random} projections. 
Based on preliminary experiments, 
our intuition is that normalization emphasizes \emph{topological} variability, 
whereas computations based on unnormalized embeddings chiefly capture \emph{geometric} variability. 
Similarly, random projections seem particularly suitable for studying variability \emph{within} individual latent spaces, 
whereas deterministic projections appear to excel at comparisons \emph{between} different latent spaces.
However, how the representational variabilities observed under each choice (or combination of choices) are related, 
and how they can be interpreted, 
merits a separate in-depth investigation.

Finally, while \ourmethod improves our understanding and handling of representational variability in latent-space models, 
we have only scratched the surface regarding its applications. 
Thus, we envisage 
\begin{inparaenum}[(1)]
	\item leveraging \ourmethod to study latent representations beyond the generative domain, 
	such as \emph{graph embeddings} and \emph{internal neural-network layers} (see \cref{apx:neural-networks} for preliminary experiments on the latter), 
	\item extending \ourmethod's reach to other areas of responsible and efficient model selection, 
	such as \emph{representational biases} and \emph{zero-shot stitching}, and 
	\item integrating \ourmethod's hyperparameter-compression and sensitivity-scoring methods into \emph{machine-learning-operations tools}. 
\end{inparaenum}
Moreover, while our initial experiments on hyperparameter-search-space compression and hyperparameter-knowledge reuse seem promising, 
additional research is necessary to understand \emph{when} and
\emph{how} we can leverage insights from one setting to inform hyperparameter-search and hyperparameter-selection strategies in other settings. 

Overall, we believe that multiverse approaches are 
essential in the development of responsible machine-learning practices, 
and that \ourmethod constitutes an important step toward establishing those approaches in the community. 
 \section*{Impact Statement}

In this paper, we introduce a topological framework 
for understanding representational variability in latent spaces, 
along with scalable practical tools to efficiently select robust and reliable machine-learning models. 
Thus, our work directly contributes to the advancement of responsible-machine-learning goals. 

\section*{Reproducibility Statement}

We make all code, data, and results publicly available. 
Reproducibility materials are available at \oururl, 
and our code is maintained at \href{https://github.com/aidos-lab/Presto}{https://github.com/aidos-lab/Presto}.

\section*{Acknowledgments}

C.C.\ is supported by \emph{Digital Futures} at KTH Royal Institute of Technology. 
B.R.\ is supported by the Bavarian state government with
funds from the \emph{Hightech Agenda Bavaria}.
 
\balance

\bibliography{references}
\bibliographystyle{icml2024}

\newpage

\appendix
\onecolumn
{\LARGE\bfseries Appendix}
\label{apx}

In this appendix, we provide the following supplementary materials. 
\begin{compactenum}[A.]
	\item \hyperref[apx:theory]{\textbf{Extended Theory.}} Definitions, proofs, and results omitted from the main text.
	\item \hyperref[apx:experiments]{\textbf{Extended Experiments.}} Additional details and experiments complementing the discussion in the main paper. 
	\item \hyperref[apx:methods]{\textbf{Extended Methods.}} Further details on properties and parts of the \ourmethod pipeline. 
	\item \hyperref[apx:Background]{\textbf{Extended Background.}} More information on the latent-space models to which \ourmethod can be applied. 
	\item \hyperref[apx:Related Work]{\textbf{Extended Related Work.}} Discussion of additional related work. 
\end{compactenum}

\section{Extended Theory}
\label{apx:theory}

In this section, we state the definitions, provide the proofs of our main theoretical results, and derive the additional results omitted from the main text. 

\begin{definition}[Bottleneck Distance $d_B$]
  Let $X$ and $Y$ be two finite metric spaces. The \textbf{bottleneck distance} between the 
  persistence diagrams of $X$ and $Y$ is denoted by $d_B(\cdot,\cdot)$ and defined as 
\begin{equation}
      d_B(D_1, D_2) = \min_{\gamma\colon D_1 \to D_2} \max_{x \in D_1} \lVert x - \gamma(x) \rVert_\infty\;,
  \end{equation}  
  for a bijection $\gamma\colon D_1\to D_2$.
\end{definition}

\begin{definition}[Vietoris--Rips Complex]
  The \textbf{Vietoris--Rips complex} of a metric space $(X,d)$ at diameter $r$,
  denoted as $\textbf{VR}(X, r)$, is defined as 
\begin{equation}
  \textbf{VR}(X, r) = \left\{ \sigma \subseteq X \mid
  \forall x, y \in \sigma, d(x,y) \leq r \right\}\;.
\end{equation}
The elements $\sigma$ of $\textbf{VR}(X, r)$ are the simplices of the complex, and they represent subsets of 
$X$ such that the pairwise distances between their elements are at most $r$.
  
\end{definition}

\textbf{Notation.}~Following the form and notation of the results in \citeauthor{krishnamoorthy_normalized_2023},
which are proved for Vietoris--Rips filtrations, we write $\mathrm{dgm}(X)$ for the Vietoris--Rips persistence diagram of a metric
space $X$, and similarly
 $d_B(X,Y)$ for the bottleneck distance between two persistence diagrams $\mathrm{dgm}(X)$ and  $\mathrm{dgm}(Y)$.

\begin{definition}[Normalized Bottleneck Distance $d_N$]
For metric spaces $(X,d_X)$ and $(Y,d_Y)$ with diameters $\diam(X), \diam(Y)$,
the \textbf{normalized bottleneck distance} is defined as 
\begin{equation}
  d_N(X,Y) = d_B\left(\frac{X}{\diam(X)},\frac{Y}{\diam(Y)}\right)\;.
\end{equation}
Moreover, as proved by \citeauthor{krishnamoorthy_normalized_2023},
$d_N$ has the favorable properties of scale invariance and stability.
\end{definition}

\begin{definition}[Johnson--Lindenstrauss Projection]
Let $(X,d_X)$ and $(Y,d_Y)$ be metric spaces. A \textbf{Johnson-Lindenstrauss projection} (JL projection) is
a linear map $f\colon X \to Y$ that satisfies the following inequality for some $0 < \varepsilon < 1$:
  \begin{equation}
      (1 - \varepsilon) \, d_X(u, v)^2 \leq d_Y(f(u), f(v))^2 
      \leq (1 + \varepsilon) \, d_X(u, v)^2 \;.
  \end{equation}
  The JL projection provides a controlled distortion of pairwise distances, 
  allowing for a significant reduction in dimensionality while approximately preserving the geometry of the original space.
  The celebrated JL Lemma guarantees the existence of such a projection~$f$.
\end{definition}

\begin{lemma}[Johnson--Lindenstrauss Persistent Homology Preservation]
Let $X \subset \reals^d $ and $\varepsilon \in (0, 1)$.
If $ f\colon\reals^d \to \reals^m$ 
is a JL linear projection, then
\begin{equation*}
d_N(X, f(X)) \leq \varepsilon\;,
\end{equation*}
where $n$, the dimension of the projection, is assumed to be 
larger than $\nicefrac{8 \ln(|X|)}{\varepsilon^2}$.
\end{lemma}

JL-linear maps represent the most efficient methods to collapse
extremely large latent spaces.
In practice, Gaussian random projections
are used, but even random orthogonal projections are sufficient
to preserve the topology.
However, to achieve even tighter
bounds on the topology of the projection at even lower dimensions, 
we can invoke more sophisticated projection methods.

\textbf{Multidimensional Scaling (MDS).} Given an input metric space $X
= \{x_1, \ldots, x_n\}$ with a metric $d_X$ and desired 
reduced dimension $m$, MDS finds a centered data set 
$\tilde{X} = \{\tilde{x}_1, \ldots, \tilde{x}_n\} \subset \reals^m$ such that
$
\sum_{i,j=1}^{n} \left(d(x_i, x_j) - \|\tilde{x}_i - \tilde{x}_j\|\right)^2
$
is minimized. This projection is achieved in two steps:
\begin{compactenum}
    \item Obtain a realization $\Phi$ of $X$ in $\reals^k$ for some $k$,
      i.e., $\Phi\colon X \to \Phi(X) \subset \reals^k$ is an isometry.
    \item Orthogonally project the realized data onto the first $m \leq d$ dominant eigenvectors of the covariance matrix $C(\Phi(X))$.
\end{compactenum}
To find a realization of $X$ in $\reals^k$, 
use the fact that $-\nicefrac{1}{2}D^{\circ 2}_X = C_nGX C_n$, where
$GX$ is the Gram matrix of $X$, and $C_n = I_n - \nicefrac{1}{n}
\mathbf{1}_n$ is the centering matrix. Since $GX$ is positive
semidefinite, it has a unique root $Z = [z_1, \ldots, z_n] \in \reals^{k \times n}$ such that $Z^T Z = GX$. 
The realization of $X$ is given by $\Phi(X) = \{z_1, \ldots, z_n\}$.
With the realization matrix $Z$, compute the \emph{singular-value
decomposition} $Z = U \Sigma V^T$, where $\Sigma = \text{diag}(\sigma_1, \ldots, \sigma_n)$ 
with $\sigma_1 \geq \ldots \geq \sigma_n$. The eigenvectors of $C(Z) = ZZ^T = U 
\Sigma^2 U^T$ are the columns of $U$, and the eigenvalues $\lambda_i$ of $C(Z)$ are 
$\sigma_i^2$. Take the first $m$ columns of $U$ to perform an orthogonal projection.

\begin{definition}[Metric Multidimensional Scaling {[}mMDS{]}]
 Let $(X,d_X) = \{x_1, \ldots, x_n\} \subset \reals^d$ be a finite metric 
 space with $GX = U \Sigma^2 U^T$ the corresponding Gram matrix, 
 and $Z = [z_1, \ldots, z_n]$ a realization of $X$ in $\reals^k$ where 
 $GX = Z^T Z$. To embed $X$ into dimension $0 < m \leq d$, we project
 onto the first $m$ dominant eigenvectors of $GX$, defined as
 $\tilde{U} = [u_1, \ldots, u_m]$. 
The mMDS reduction $P^{(m)}_X \colon X \to \reals^m$
  is the map
\begin{equation}
  P^{(m)}_X(x_i) = \tilde{x}_i = \text{proj}_{\text{Im}(\tilde{U})}(z_i)\;.  
\end{equation}  
\end{definition}

\begin{lemma}[mMDS Homology Preservation]
Let $ X \subset \reals^d$ and $0<m \leq d$.
Then $ P^{(m)}_X\colon X \to \reals^m$ preserves the homology
of $X$ according to the following bound on $d_N$:  
  \[
  d_N(X, P^{(m)}_X) \leq 
  \frac{2\sqrt{2}}{\diam(P^m_X(X))}
  \sqrt[4]{\frac{(\sum_{1}^{m}\lambda_i^2)(\sum_{m+1}^{d}\lambda_i^2)}{(\sum_{1}^{d}\lambda_i^2)}}\;,
  \]
  where $\lambda_i$ is defined to be the $i^{th}$ eigenvalue of the
  covariance matrix $C(X)$.
\end{lemma}

\begin{lemma}[bi-Lipschitz Homology Preservation]
  For metric spaces $(X,d_X),(Y,d_Y)$, let $f\colon X\to Y$ be a $k$-bi-Lipschitz map.
  The change in $d_N$
 can be bounded as
\begin{equation}
   d_N(X,Y) = \left|\frac{k^2-1}{k}\right| \frac{\diam(X)}{\diam(Y)} \;.
 \end{equation}
\end{lemma}

Before we can prove general statements about the preservation of
our topological distance under projections, we need to prove some
supporting lemmas concerning the relationships between various
topological distances.

\begin{lemma}[Bottleneck-Distance Bound]
  Let $p \in \reals_{> 0}$ and  $W_p$ denote the $p$th Wasserstein
  distance between persistence diagrams. Then the \emph{bottleneck
  distance} $d_B$ constitutes an upper bound for $W_p$, i.e., $W_p \leq
  d_B$.
\label{lem:Bottleneck bound}
\end{lemma}

\begin{proof}
We prove a more general statement.
For $p \leq q$, the function $\phi(x) := x^\frac{q}{p}$ is
  \emph{convex}.
Thus, for any function~$f$ and any probability measure $\mu$, we have
\begin{align}
    \left(\int f^p \mathrm{d}\mu\right)^{\frac{1}{p}} &= \phi\left(\left(\int f^p \mathrm{d}\mu\right)^{\frac{1}{p}}\right)^{\frac{1}{q}}\\
                                            &\leq \mleft(\int \phi\left(f^p\right) \mathrm{d}\mu\mright)^{\frac{1}{q}}
                                            && \text{(by Jensen's inequality)}\\
                                            &= \left(\int f^q \mathrm{d}\mu \right)^{\frac{1}{q}}.
  \end{align}
As this result holds for \emph{general} functions~$f$, it particularly
  holds for the cost functions used in the calculation of the
  Wasserstein distance between persistence diagrams. 
Finally, we have $\lim_{p \to \infty} W_p = d_B$~(some works even use the
  suggestive notation $W_\infty$ to denote the bottleneck distance),
  concluding the proof.
\end{proof}

Continuing in a similar vein, we turn to the analysis of the norms of
\emph{persistence landscapes}.

\begin{lemma}[Landscape-Norm Bound]
  Let $p \in \reals_{> 0}$ and  $\|\cdot\|_p$ denote the $p$th persistence
  landscape norm~\citep[Section~2.4]{bubenik2015statistical}.
Then the \emph{infinity norm} $\|\cdot\|_\infty$ is an upper bound for
  $\|\cdot\|_p$, i.e., $\|\lambda\|_p \leq \|\lambda\|_\infty$ for all
  persistence landscapes~$\lambda$.
\label{lem:Landscape bound}
\end{lemma}

\begin{proof}
  The persistence landscape norm is simply an $L^p$ norm, calculated for
  a specific class of piecewise-linear functions, i.e., the persistence
  landscapes. We prove a more general result holding for \emph{all}
  functions~$f$ that satisfy a certain integral property. Specifically,
  we assume that $f \in L^q$, i.e., the absolute value of $f$, raised to
  the $q$th power, has a finite integral. This holds for a large class
  of functions, and for \emph{all} persistence landscapes in particular.
  Since $|f|^p \leq |f|^q$ for $p \leq q$, this implies that $f \in
  L^p$. Now let $a \coloneq \nicefrac{q}{p}$ and $b \coloneq
  \nicefrac{q}{q-p}$. Since $\nicefrac{1}{a} + \nicefrac{1}{b} = 1$, the
  numbers $a$ and $b$ are \emph{Hölder conjugates} and we may apply
  Hölder's inequality with $g = \mathds{1}$, the function that is
  identically to~$1$ over the domain of~$f$, yielding
\begin{equation}
    \|f^p \cdot \mathds{1}\|_1 \leq \|f^p\|_a \cdot \|\mathds{1}\|_b\;.
  \end{equation}
The left-hand side is equivalent to $\|f\|_p^p$, while the right-hand side,
  following the definition of the norm, evaluates to
\begin{align}
    \|f^p\|_a &= \int\left(|f^p|^\frac{q}{p}\mathrm{d}\mu\right)^{\frac{p}{q}}\\
              &= \|f\|_q^p.
  \end{align}
By definition of the norm, the other factor
  satisfies $\|\mathds{1}\|_b \geq 1$.
Putting this together, we have
\begin{align}
    \|f\|_p^p &\leq \|f\|_q^p\\
    \therefore \|f\|_p &\leq \|f\|_q\;.
  \end{align}
As a consequence, the persistence-landscape norms are bounded
  similarly to the Wasserstein distances. 
  Noticing that $\lim_{p \to \infty}
  \|f\|_p = \|f\|_\infty$ concludes the proof.
\end{proof}

The immediate consequence of \cref{lem:Bottleneck bound} and
\cref{lem:Landscape bound} is that \emph{any} bound calculation of our
topological distances is, eventually, upper-bounded by the bottleneck
distance between persistence diagrams. While such a bound is by its very
nature potentially rather coarse, the bottleneck distance is suitable as
the most extreme upper topological bound since it is known to be itself
bounded by the geometrical variation between the two spaces from which
the respective persistence diagrams arise. More precisely, we have
\begin{equation}
  d_B(X, Y) \leq 2 d_{\text{GH}}(X, Y)\;,
\end{equation}
where $d_{\text{GH}}$ denotes the \emph{Gromov--Hausdorff distance}
between~$X$ and~$Y$~\citep{Chazal14a}.
Thus, any topological variation, regardless of the distance metric,
is upper-bounded by the geometrical variation between two spaces.

\thmMMSPreservation*
\begin{proof}
  Following the preceding discussion, as well as  \cref{lem:Bottleneck bound}
  and \cref{lem:Landscape bound}, it is sufficient to phrase the desired
  inequality in terms of the bottleneck distance between two spaces.
Let $E_i, E_j$ denote the respective embeddings, and let~$f$ be
  a projector.
Starting from the left-hand side, and using the triangle inequality,
  we get
\begin{align}
    d_B\left(f\left(E_i\right), f\left(E_j\right)\right) &\leq d_B\left(f\left(E_i\right), E_i\right) + d_B\left(E_i, f\left(E_j\right)\right)\\
      &\leq d_B\left(f\left(E_i\right), E_i\right) + d_B\left(E_i, E_j\right) + d_B\left(E_j, f\left(E_j\right)\right)\\
      & \leq \MMS[i, j] + d_B\left(f\left(E_i\right), E_i\right) + d_B\left(E_j, f\left(E_j\right)\right).
  \end{align}
The last two terms on the right-hand side depend on the selected
  projector function. By \cref{eq:topological-loss}, we see that their
  sum is upper-bounded by $2\topologybound$, concluding the proof.
\end{proof}

\cref{thm:ClusteringPreservation} might not be entirely satisfying
because $\topologybound$ still has an underlying dependency on the
projector. However, there are no general bounds available unless we
choose a specific projector. Notice that for many classes of projectors,
including metric MDS, PCA, and random projections, suitable bounds for
$\topologybound$ itself may be obtained.

\thmSensitivityPreservation*
\begin{proof}
  For the first part of the statement, we notice that
  $\|\landscape(\embedding_i)\| = d(\landscape(\embedding_i),
  \emptyset)$, i.e., the distance of a given persistence landscape to
  the empty landscape. Thus,
\begin{align}
    \|\landscape(\embedding_i)\| &\leq d(\landscape(\embedding_i), \landscape(\projection_i)) + d(\landscape(\projection_i), \emptyset) && \text{(triangle inequality)}\\
                                 &=d(\landscape(\embedding_i), \landscape(\projection_i)) + \|\landscape(\projection_i)\| && \text{(by definition of the norm)}\\
                                 &\leq \|\landscape(\projection_i)\| + \topologybound && \text{(by definition of $\topologybound$)}.
  \end{align}
Applying the same argument to $\|\landscape(\projection_i)\|$ concludes the proof of this part.
As for the second part, we recall that we want to bound
  $\cardinality{\text{\ourvariance}(\landscapes_\MMS) - \text{\ourvariance}(\landscapes_{\MMS^\projdim})}$.
To this end, we need to find a suitable bound for the second term, i.e.,
  the term that deals with \emph{projected} embeddings.
Here, we note that each summand is of the form
  $(\|\landscape(\projection_i)\| - \mu_{\projections})^2$, for which we can
  obtain a bound using \cref{eq:Landscape norm change} as
\begin{align}
    \|\landscape(\projection_i)\| - \mu_{\projections} &\leq \|\landscape(\embedding_i)\| + \topologybound - \mu_{\projections}\\
                                                       &\leq \|\landscape(\embedding_i)\| + \topologybound - \mu_{\embeddings} + \topologybound\;,
  \end{align}
with the second inequality arising from the lower bound of \cref{eq:Landscape norm change}.
$\cardinality{\text{\ourvariance}(\landscapes_\MMS) - \text{\ourvariance}(\landscapes_{\MMS^\projdim})}$.
Putting this together and calculating the differences per term, we get
\begin{equation}
    \cardinality{\text{\ourvariance}(\landscapes_\MMS) - \text{\ourvariance}(\landscapes_{\MMS^\projdim})}
    \leq \frac{4 \topologybound}{N} \sum_{i = 1}^{N}\left(\|\landscape(E_i)\| - \mu_\embeddings\right) + \frac{\left(2\topologybound\right)^2}{N}\;.
  \end{equation}
\end{proof}

\section{Extended Experiments}
\label{apx:experiments}

In this section, we provide additional details and experiments complementing the discussion in \cref{Experiments}, presenting
\begin{compactenum}[1.]
	\item extended experiments in our \textbf{VAE hyperparameter multiverse} and our \textbf{\bvae implementation multiverse}, 
	\item extended experiments in our \textbf{transformer multiverse}, including a multiverse analysis of the \textbf{\ourmethod pipeline}, 
	\item preliminary experiments exploring \ourmethod's use as a dissimilarity measure for \textbf{neural-network representations}, and 
	\item a multiverse analysis of \textbf{non-linear dimensionality-reduction} methods. 
\end{compactenum}
We make all our code, data, and results available at \oururl. 
Our code is maintained at \href{https://github.com/aidos-lab/Presto}{https://github.com/aidos-lab/Presto}.

\subsection{Extended Experiments in VAE Multiverses}
\label{apx:experiment-descriptions}

Here, we provide further details on the configuration of our VAE multiverses, as well as additional experiments complementing our discussion in the main text. 

\subsubsection{The \vae Hyperparameter Multiverse}
\label{apx:experiment-hyperparameter-multiverse}

\begin{table}[t]
	\centering
	\begin{tabular}{rrrr}
		\toprule
		VAE&$\params_i$&Values&Description\\\midrule
		\multirow{3}{*}{\bvae}&$\beta$&$1,4,16,64$&Recon Bias\\ 
		&$\gamma$&$500,750,1\,000$&KLD Bias\\ 
		&$\ell$&$\text{B},\text{H}$&Loss Variations\\
		\midrule
		\multirow{3}{*}{\info}&$\beta$&$1,5,10$&Recon Bias\\
        &$\alpha$&$-5,-2,-0.5,0$&KLD Bias\\
		&$\kappa$&$\text{imq},\text{rbf}$&MMD Kernel\\
		\midrule
		\multirow{3}{*}{\wae}&$\lambda$&$10,20,50,100$&MMD Prior Bias\\
		&$\nu$&$1,2,3$&Kernel Width\\
		&$\kappa$&$\text{imq},\text{rbf}$&MMD Kernel\\
		\bottomrule
	\end{tabular}
	\caption{\textbf{VAE Hyperparameter Multiverse.}
		For each of three VAEs, we explore the product of all varied-parameter 
		values across five datasets for two sample sizes ($100\%$ and $50\%$ of the training set), 
		for a total of $3\cdot 5 \cdot 4 \cdot 2 \cdot 3 \cdot 2 = 720$ configurations.
		In particular, we get $48$ configurations \emph{per architecture}, which we use to 
		evaluate sensitivity in  \cref{tab:vae-local-sensitivity,tab:vae-global-sensitivity},
		detect outliers in \Cref{fig:vae-landscape-norms}, and compare multiverses in  \Cref{fig:reusing-hyperparameters}. 
		For \bvae loss variations, H is the original implementation from
		\citet{higgins2017betavae}, and B stems from \citet{burgess2018understanding}.
	}\label{tab:vae-hyperparameter-multiverse}
\end{table}

\begin{table}[t]
	\centering
	\
	\begin{tabular}{lcccccc}
		\toprule
		\multirow{2}{*}{Parameters} & \multicolumn{5}{c}{Datasets} \\
		\cmidrule(lr){2-6}
		& \celeba & \cifar & \dsprites & \fashionmnist & \mnist \\
		\midrule
		Hidden Dimensions & $(32, 64, 128, 256, 512)$ &$(32, 64, 128)$ &$(8, 16)$ & $(32, 64)$ & $(32, 64)$ \\
		Latent Dimension & $50$ & $25$ & $25$ & $10$ & $10$ \\
		Batch Size & $128$ & $128$ & $128$ & $64$ & $64$ \\
		\bottomrule
	\end{tabular}
	\caption{\textbf{Default implementation choices in the VAE hyperparameter multiverse}. 
		We show our fixed implementation parameters,
		over which we vary the algorithmic parameters listed in \cref{tab:vae-hyperparameter-multiverse}.  
		Each model was trained using 
		an \textsc{Adam} optimizer with a learning 
		rate of $0.001$ over $30$ epochs.
	}\label{tab:vae-hyperparameter-defaults}
\end{table} 

\begin{table}[b]
	\centering
	\begin{tabular}{rrrrr}
		\toprule
		$\beta$ & Parameter & Values\\\midrule
		\multirow{3}{*}{$\{2,16,64\}$} 
		& Batch Sizes ($b$) & $8,16,32,64,128,256$\\
		& Learning Rates ($l$) & $0.002,0.004,0.008,0.016,0.032,0.064$ \\ 
		& Training Sample Sizes ($s$) & $0.5,0.6,0.7,0.8,0.9,1$ \\ 
		\bottomrule
	\end{tabular}
	\caption{\textbf{\bvae Implementation Multiverse.}
		For \bvae, we explore the relation between the $\beta$ hyperparameter and various implementation choices. This table
		contains our choices for batch size ($b$), learning rate ($l$), and sample size ($s$).
		Our choices for train-test split ($t$) and hidden dimensions ($h$), are explained in 
		\cref{tab:hidden-dimensions-all-datasets,apx:vae-implementation-multiverse}, respectively.
	}\label{tab:vae-implementation-multiverse} 
\end{table}

As previewed in the main text, our VAE hyperparameter multiverse 
investigates the hyperparameter space for three commonly cited 
autoencoder architectures, namely
\begin{inparaenum}[(1)]
	\item \bvae \cite{higgins2017betavae}, 
	\item \info \cite{zhao2018infovae}, and
	\item \wae \cite{tolstikhin2019wasserstein},
\end{inparaenum}
covering hyperparameters ranges
that appear commonly in the literature, as well as
widely used open-source implementations. The explicit
values and brief descriptions of the grid searches that 
determine our multiverse are detailed in \cref{tab:vae-hyperparameter-multiverse}.

We select three hyperparameters for each architecture, 
combining their searches to generate $24$ unique configurations.
Each model was trained using 
a random $[0.6,0.3,0.1]$ train/validation/test split 
for each of our five datasets.
Furthermore, to understand the variability in latent structure 
under different cardinalities, we train models with
$100\%$ ($0.6$) and $50\%$ ($0.3$) training-set sizes, keeping 
validation- and test-set cardinalities fixed.

Although our multiverse approach supports varying 
algorithmic, implementation, and data choices,
we found it prudent to design an 
environment that demonstrates the latent 
variability of algorithmic choices. 
Naturally, this required fixing \emph{some}
parameters across runs (e.g., our train/test/split ratios). 
See \cref{tab:vae-hyperparameter-defaults} for 
the values of our fixed implementation choices.

\subsubsection{The \bvae Implementation Multiverse}
\label{apx:vae-implementation-multiverse}
Complementing our VAE hyperparameter multiverse, 
we design a multiverse to explore how implementation choices, i.e.,
\begin{inparaenum}[(1)]
	\item batch size $b$,
	\item hidden dimensions $h$, 
	\item learning rate $l$,
	\item sample size $s$, 
	\item and train-test split $t$
\end{inparaenum}  
can affect the latent representations of variational autoencoders. We focus 
on \bvae, varying the aforementioned parameters. Our exact choices are detailed in 
\cref{tab:vae-implementation-multiverse} and \cref{tab:hidden-dimensions-all-datasets}.

\clearpage

Our train-test splits iterate over five random train/validation/test splits at a
fixed size of $[0.6,0.3,0.1]$. Moreover, acknowledging interplay between 
algorithmic and implementation choices, we train each implementation configuration 
over 3 different choices for $\beta\in\{2,16,64\}$, 
for the datasets that we also considered 
in the hyperparameter multiverse, i.e., 
\begin{inparaenum}[(1)]
	\item \celeba,
	\item \cifar,
	\item \dsprites,
	\item \fashionmnist, and
	\item \mnist.
\end{inparaenum}

\begin{table}[t]
	\centering \small
	\begin{tabular}{ccccccc}
		\toprule
		Hidden Layers & \celeba & \cifar & \dsprites & \fashionmnist & \mnist \\
		\midrule
		\multirow{5}{*}{2}
		& $(16, 32)$ & $(8, 16)$ & $(8, 16)$ & $(8, 16)$ & $(8, 16)$ \\
		& $(32, 64)$ & $(16, 32)$ & $(16, 32)$ & $(16, 32)$ &$(16, 32)$ \\
		& $(64, 128)$ & $(32, 64)$ & $(32, 64)$ & $(32, 64)$ &$(32, 64)$ \\
		& $(128, 256)$ & $(64, 128)$ & $(64, 128)$ & $(64, 128)$ & $(64, 128)$ \\
		\midrule
		\multirow{2}{*}{3} & $(32, 64, 128)$ & $(8,16,32)$ &$(16,32,64)$ & $(16,32,64)$ & \\
		&  & $(32,64,128)$ & & & \\
		\midrule
		\multirow{1}{*}{4} & & &  $(16,32,64,128)$ & $(16,32,64,128)$ & $(16,32,64,128)$ \\
		\midrule
		\multirow{1}{*}{5}  & $(32, 64, 128, 256, 512)$  & & & & \\ 
		\bottomrule
	\end{tabular}
	\caption{\textbf{Implementation choices in the \bvae implementation multiverse: Hidden layers for our datasets}. Given the varying complexity of each dataset in 
	our \bvae implementation multiverse, we vary the complexity of the layer as seen in the literature and popular 
	\textsc{Pytorch} autoencoder frameworks.}
	\label{tab:hidden-dimensions-all-datasets}
\end{table}

\subsubsection{\ourmethod Sensitivities}
\label{apx:vae-sensitivities}

Our sensitivity scores are based on the variance in the structure of a
\emph{distribution} of latent spaces, as measured by \ourmethod. Though this is not restricted to 
multiverse analysis (i.e., the variation in landscape norms can be applied to any distribution 
of embeddings), the multiverse treatment of different algorithmic, implementation, and data choices 
results naturally in a collection of latent representations. Thus, we tailor our scores 
specifically to understanding variation within a multiverse.
As stated in the main text, we define three different 
variations of \ourmethod sensitivity (\oursensitivity) in \cref{def:latent-space-parameter-sensitivity} that make use of 
\cref{def:latent-space-landscape-variance}. 
We repeat these definitions here for convenience.

Given a multiverse \multiverse, fix a model dimension $i$,
	and define an equivalence relation $\sim_i$ such that $\params' \sim_i \params'' \Leftrightarrow \params'_j = \params''_j$ for all $\params', \params''\in\multiverse$ and $j \neq i$, 
	yielding $\neqclasses_i$ equivalence classes~$\eqclasses_i$.
	
	The \textbf{\emph{individual} $\ourmethod^\distancenorm$ sensitivity}
	of equivalence class $\eqclass\in\eqclasses_i$ in \multiverse is

\begin{equation*}
	\text{\ourcoordinatesensitivity}_\topodim^\distancenorm(\eqclass\mid\multiverse) \coloneq \sqrt{\text{\ourvariance}_\topodim^\distancenorm(\landscapes[\eqclass])}\;,
\end{equation*}
where $\landscapes[\eqclass]\subset\landscapes$ is the set of
	landscapes associated with models in equivalence class \eqclass. 
Aggregating over all equivalence classes in $\eqclasses_i$, we obtain
	the \textbf{\emph{local} $\ourmethod^\distancenorm$ sensitivity} of
	\multiverse in model dimension $i$ as 

\begin{equation*}
	\text{\oursensitivity}_\topodim^\distancenorm(\multiverse \mid i) \coloneq \sqrt{\frac{1}{\neqclasses_i}\sum_{\eqclass\in\eqclasses_i}\text{\ourvariance}_\topodim^\distancenorm(\landscapes[\eqclass])}\;.
\end{equation*} 

Finally, aggregating over all $\nparams = \cardinality{\params}$
	dimensions of models in \multiverse yields the \textbf{\emph{global}
		$\ourmethod^\distancenorm$ sensitivity} of \multiverse, i.e., 

\begin{equation*}
	\text{\oursensitivity}_\topodim^\distancenorm(\multiverse) \coloneq \sqrt{\frac{1}{\nparams}\sum_{i\in[\nparams]}\frac{1}{\neqclasses_i}\sum_{\eqclass\in\eqclasses_i}\text{\ourvariance}_\topodim^\distancenorm(\landscapes[\eqclass])}\;.
\end{equation*}

Recall from \cref{def:latent-space-landscape-variance} that $p$ and $h$
represent the $\distancenorm$-norm for landscapes and the homology dimension, respectively. 
In our experiments, we consider up homology features up to second order ($h=2$), which preserves 
the scalability of our pipeline while still capturing descriptive higher-dimensional 
topological information. Additionally, we default to $\distancenorm = 2$, 
understanding nicely the theoretical trade-offs for different $\distancenorm$-norms 
(cf.~\cref{lem:Landscape bound}).

\begin{table*}[t]
	\centering
\begin{tabular}{llrrrrrrrr}
\toprule
 &  & $\mu_{\lVert\mathcal{L}\rVert}$ & $\sigma_{\lVert\mathcal{L}\rVert}$ & $\mu_{\text{PS}}$ & $\sigma_{\text{PS}}$ & $\sigma_{\lVert\mathcal{L}\rVert}/\mu_{\lVert\mathcal{L}\rVert}$ & $\sigma_{\text{PS}}/\mu_{\text{PS}}$ & $\mu_{\text{PS}}/\mu_{\lVert\mathcal{L}\rVert}$ & $\sigma_{\text{PS}}/\sigma_{\lVert\mathcal{L}\rVert}$ \\
VAE & $\theta_i$ &  &  &  &  &  &  &  &  \\
\midrule
\multirow[c]{3}{*}{\bvae} & $\beta$ & 0.0171 & 0.0012 & 0.0067 & 0.0012 & 0.0714 & 0.1762 & 0.3899 & 0.9615 \\
 & $\gamma$ & 0.0171 & 0.0017 & 0.0065 & 0.0016 & 0.1021 & 0.2453 & 0.3807 & 0.9148 \\
 & $\ell$ & 0.0171 & 0.0015 & 0.0065 & 0.0015 & 0.0893 & 0.2364 & 0.3819 & 1.0105 \\
\cline{1-10}
\multirow[c]{3}{*}{\info} & $\alpha$ & 0.0182 & 0.0063 & 0.0072 & 0.0044 & 0.3475 & 0.6176 & 0.3953 & 0.7026 \\
 & $\beta$ & 0.0182 & 0.0040 & 0.0081 & 0.0047 & 0.2191 & 0.5815 & 0.4461 & 1.1841 \\
 & $\kappa$ & 0.0182 & 0.0043 & 0.0081 & 0.0047 & 0.2344 & 0.5820 & 0.4421 & 1.0978 \\
\cline{1-10}
\multirow[c]{3}{*}{\wae} & $\lambda$ & 0.0185 & 0.0008 & 0.0075 & 0.0005 & 0.0405 & 0.0610 & 0.4063 & 0.6111 \\
 & $\nu$ & 0.0185 & 0.0011 & 0.0075 & 0.0010 & 0.0573 & 0.1306 & 0.4022 & 0.9168 \\
 & $\kappa$ & 0.0185 & 0.0013 & 0.0074 & 0.0013 & 0.0690 & 0.1705 & 0.3978 & 0.9824 \\
\bottomrule
\end{tabular}

 	\caption{\textbf{Local \ourmethod sensitivity in the VAE hyperparameter multiverse.}
		Along with the average and standard deviation of the landscape norms for the 
		latent representations associated with each architecture ($\mu_{\lVert\landscapes\rVert},
		\sigma_{\lVert\landscapes\rVert}$), we display the local sensitivities for each
		of the main parameters searched (see \cref{tab:vae-hyperparameter-multiverse}
		for our parameter list). Recall that local \oursensitivity is an average over 
		different equivalence classes that partition \multiverse,
		denoted ($\mu_{\oursensitivity}$). Given the 
		standard partition sizes in this multiverse 
		(equal number of hyperparameter configurations per architecture),
		we can also take standard deviations ($\sigma_{\oursensitivity}$) and other 
		derivatives, further describing the relative variability between 
		architectures. We see that different \vae architectures
		have varying levels of sensitivities and robustness across their hyperparameter spaces,
		which should be investigated thoroughly when using them as generative models.}\label{tab:vae-local-sensitivity}
\end{table*}

\begin{table*}[t]
	\centering
	\begin{tabular}{lrrrrrrrr}
\toprule
 & $\mu_{\lVert\mathcal{L}\rVert}$ & $\sigma_{\lVert\mathcal{L}\rVert}$ & $\mu_{\text{PS}}$ & $\sigma_{\text{PS}}$ & $\sigma_{\lVert\mathcal{L}\rVert}/\mu_{\lVert\mathcal{L}\rVert}$ & $\sigma_{\text{PS}}/\mu_{\text{PS}}$ & $\mu_{\text{PS}}/\mu_{\lVert\mathcal{L}\rVert}$ & $\sigma_{\text{PS}}/\sigma_{\lVert\mathcal{L}\rVert}$ \\
VAE &  &  &  &  &  &  &  &  \\
\midrule
\bvae & 0.0171 & 0.0015 & 0.0066 & 0.0014 & 0.0876 & 0.2193 & 0.3842 & 0.9623 \\
\info & 0.0182 & 0.0049 & 0.0078 & 0.0046 & 0.2670 & 0.5937 & 0.4278 & 0.9948 \\
\wae & 0.0185 & 0.0011 & 0.0075 & 0.0009 & 0.0556 & 0.1207 & 0.4021 & 0.8368 \\
\bottomrule
\end{tabular}

 	\caption{\textbf{Global \ourmethod sensitivity in the VAE hyperparameter multiverse.}
		Here, we aggregate the local sensitivities into a global sensitivity score, 
		analyzing the representational variability
		across $48$ unique latent representations for each \vae architecture in the multiverse.
		Again, we provide additional derivative statistics to give context for the scale of these scores.
		These values echo our findings described in the main paper: 
		\info has the highest representational variability in the
		\vae hyperparameter multiverse.}
		\label{tab:vae-global-sensitivity}
\end{table*}

\begin{figure}[H]
	\centering
	\includegraphics[width=\linewidth]{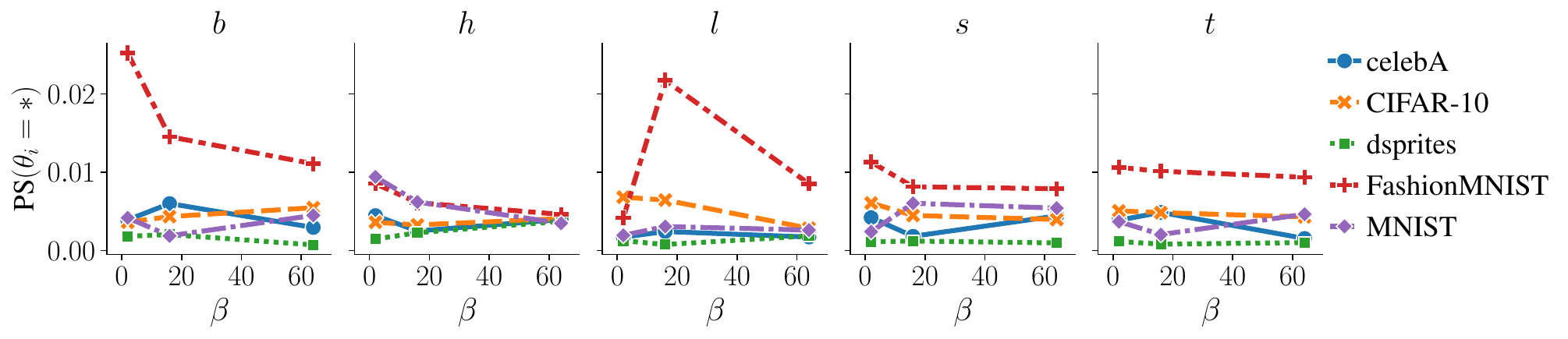}
	\caption{\textbf{\ourmethod sensitivity in the \bvae implementation multiverse.}
		We show the \ourmethod sensitivity scores for the \bvae as a function of $\beta$ 
		when varying (from left to right) batch size $b$, 
		hidden dimensions $h$, 
		learning rate $l$, 
		sample size \nsamples, 
		and the train-test split $t$.
		There is no consistent relationship between the choice of $\beta$ and the sensitivity of the latent space to implementation-parameter choices.
	}\label{fig:vae-implementation-sensitivity}
\end{figure}

Against this background, 
we provide overviews of local and global \ourmethod sensitivities and related statistics in the VAE hyperparameter multiverse in \cref{tab:vae-local-sensitivity,tab:vae-global-sensitivity}, 
along with an analysis of \ourmethod sensitivities in the \bvae implementation multiverse in \Cref{fig:vae-implementation-sensitivity}.
The overview tables confirm our impression from the main paper that \info has the highest representational variability among our VAE models, 
whereas \Cref{fig:vae-implementation-sensitivity} highlights the complex interplay between hyperparameter and implementation choices.

\clearpage

Extending our exploration of how training affects variability in latent-space models (cf. \Cref{fig:vae-landscape-norms}),
we further investigate the sensitivity of random initializations in a small multiverse of \bvae models trained on the \mnist dataset. 
In particular, we fix all parameters except for the random seed that initializes
the \bvae, leading to eight models in the multiverse 
(see \Cref{tab:random-init-config} for the full configurations). 
We use \ourmethod to
assess how much structural variability is induced by random initializations alone, 
and compare this to the variability we observe in the latent spaces of our trained models. 
To this end, we compute the \ourmethod sensitivity over eight untrained embeddings and eight trained embeddings based on a fixed training set from \mnist. 
We find that the sensitivity of the initial embeddings is higher ($\text{\oursensitivity} = 0.0188$) than that of the trained embeddings ($\text{\oursensitivity} = 0.0076$).

To complement this analysis, 
we additionally compute the multiverse metric spaces (MMSs) for the multiverses arising from our initialized and trained embeddings, 
visualizing the pairwise \ourmethod distances between the universes in each MMS in \Cref{fig:random-initializations}. 
Again, we find that the topological variation of the initial embeddings is \emph{higher} than that of their trained counterparts, 
indicating a convergence to similar topological features over training.
\ourmethod's invariance to scaling (when using normalization) allows us to analyze this convergence despite geometric differences in coordinate systems between the trained latent representations. 
We also observe that \ourmethod reveals a seed resulting in an anomalous latent representations (here: seed $2^6$), 
highlighting the importance of tools like \ourmethod that can understand distributions of latent representations and reassess the impact of choices like random initialization that are often overlooked in practice.

\begin{table}[t]
	\centering
	\begin{tabular}{lc}
			\toprule
			Parameter & Value \\
			\midrule
			Dataset & \mnist \\
			Batch Size & 64 \\
			\midrule
			Model & \bvae \\
			\textit{Seeds*} & 0, $2^1, 2^2, 2^3, 2^4, 2^5, 2^6, 2^7$ \\
			Latent Dimension & 5 \\
			Hidden Dimensions & (8, 16) \\
			$\beta$ & 4 \\
			$\gamma$ & 1000.0 \\
			Loss Type & H \\
			\midrule
			Optimizer & Adam \\
			Learning Rate & 0.001 \\
			\bottomrule
	\end{tabular}
	\caption{\textbf{Model configurations for experiment with random initializations.} 
		To investigate how training affects variability in latent-space models, 
		we train eight \bvae models on a fixed training set of \mnist. 
		The models only differ in their random initializations (\textit{Seeds}*).
	}
	\label{tab:random-init-config}
\end{table}

\begin{figure}[t]
	\centering
	\begin{subfigure}{0.25\linewidth}
		\centering
		\includegraphics[height=3cm]{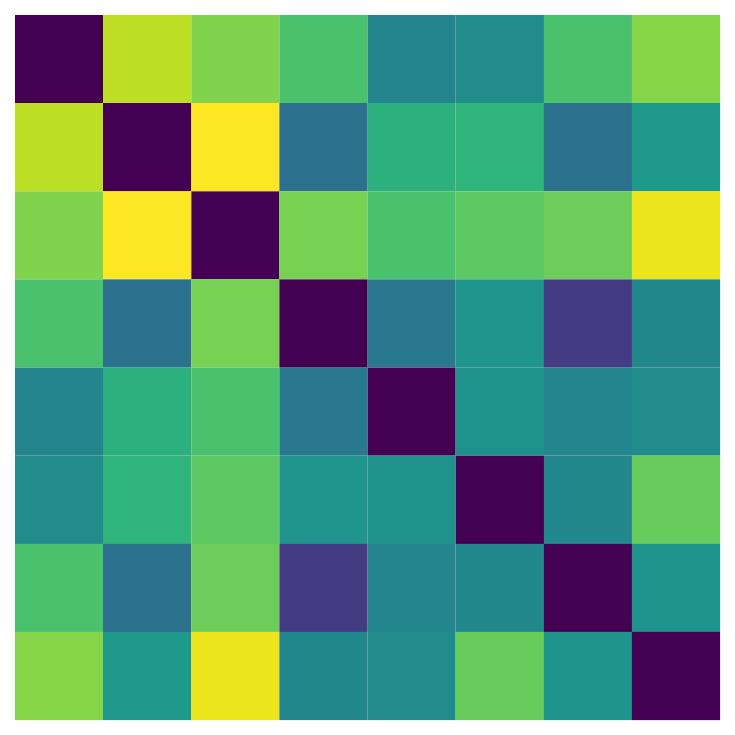}
		\subcaption{Untrained Models}\label{subfig:untrained}
	\end{subfigure}
	\begin{subfigure}{0.25\linewidth}
		\centering
		\includegraphics[height=3cm]{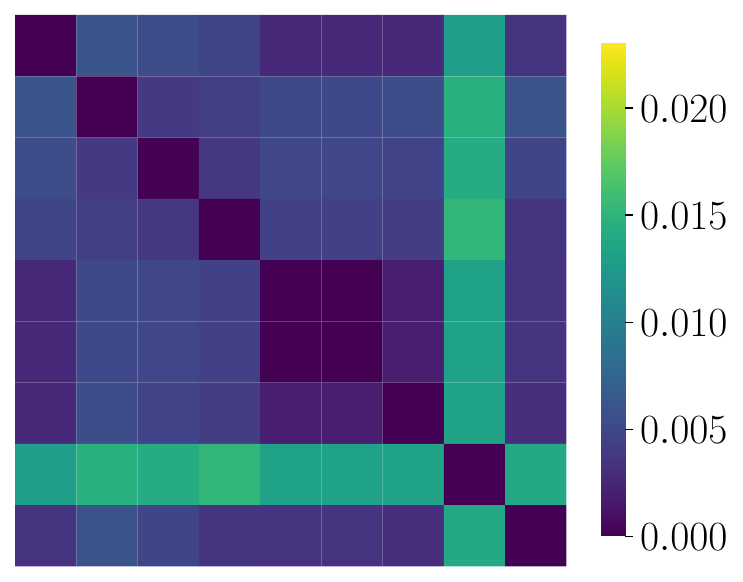}
		\subcaption{Trained Models}\label{subfig:trained}
	\end{subfigure}
	\caption{\textbf{Examining latent representations pre- and post-training.} 
		We juxtapose the pairwise \ourmethod distances of untrained latent representations that arise from identical \bvae model configurations trained on \mnist with different random seeds~(a) 
		with the pairwise \ourmethod distances of latent representations arising from these configurations after training~(b). 
		We find that training decreases the structural variability among randomly seeded models, implying that models agree about key topological features. 
		Moreover, \ourmethod is able to detect anomalous seeds, 
		i.e., seeds resulting in latent representations that, when trained, constitute structural outliers.
		}
		\label{fig:random-initializations}
\end{figure}

\clearpage

\subsubsection{\ourmethod's Relation to Geometric and Generative Similarity}
\label{apx:exp:geomgen}

Having revisited \ourmethod sensitivities, 
we now expand our discussion around \Cref{fig:geometric-generative}, 
where we asked what geometric and generative signal exactly is picked up by our framework.
We begin by describing our experimental setup in more detail and then move on to discuss \ourmethod's relation to various notions of geometric and generative similarity. 

\textbf{Experimental Setup.} Let $\embedding = \probe(X)$ for
$\probe\in \multiverse$ be a latent space in our VAE multiverse, where
$X$ is a training set of images. 
Fix a subspace cardinality ($N$), such 
that we can choose a random sample from the training set $S\subseteq X$ such that 
$|S| = N$. We then map $S$ into the latent space, obtaining $\probe(S) \subseteq \embedding$,
which will be our objects of comparison, i.e., 
$\subspaces \coloneq \{ \probe(S) \mid \probe \in \multiverse\}$, where 
$|\subspaces| = n_d$ is the number of random draws. 
This results in a set of random latent subspaces,
used to compute the correlations displayed in \Cref{fig:geometric-generative}.

\textbf{Geometric Similarity.} 
To study the similarity in geometric structure between
latent spaces, we can use a number of tools from representation learning. Recognizing
the plethora of well-studied approaches in the field, and acknowledging
the myriad interesting tools we could use to further investigate this phenomenon in future work, 
for the purposes of our current exposition, we take relatively simple approach. 
We opt to endow each $\subspace \in \subspaces$
with a metric, such that we can compare metric spaces between $\subspace,\probe'(S) \in \subspaces^2$
using the Pearson distance between their matrix representations.
Note that these are indeed \emph{metric subspaces} of the original embeddings.
By aggregating the observed behavior over a 
large number of random draws, we obtain a computable baseline for assessing the 
geometric capabilities of \ourmethod (working \emph{without} normalization).
Given that the $\alpha$-complexes leveraged by \ourmethod default to using 
Euclidean distances between points, we also use this to produce a (pairwise distance) matrix representation
of the elements of $\subspaces$. 
The results are displayed in the left panel of \Cref{fig:geometric-generative}. 
Additionally, given the utility of \emph{cosine similarity} in the study of latent spaces, we 
compute the Pearson distance between the pairwise cosine similarity matrices for $\subspace,\probe'(S)$,
to assess \ourmethod's relation to a different geometric signal in the right panel of \Cref{fig:generative-cosine}.

\begin{figure}[h]
	\centering
	\includegraphics[width=0.475\linewidth]{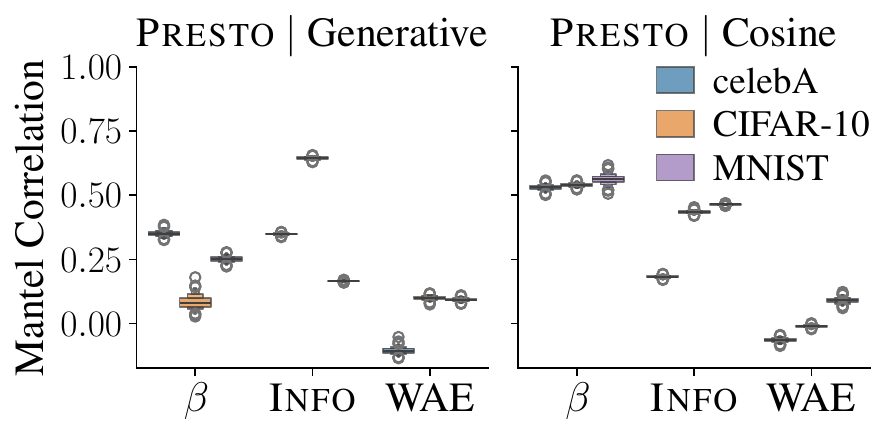}
	\caption{\textbf{\ourmethod's correlation with generated images and cosine distances.}
	We assess \ourmethod's ability to detect generative 
	(dis)similarity between models by measuring its correlation with batch MSE loss for random 
	subsets of images, generated by perturbing aligned latent coordinates (left), 
	as well as its correlation with the Pearson distance between
	random subspaces represented by their pairwise cosine similarity matrices (right).
	Some \vae architectures show potential to have unsupervised generative properties described by the topological and
	 geometric properties of their latent spaces.}
 \label{fig:generative-cosine}
\end{figure}

\textbf{Generative Similarity.} As established in \Cref{fig:geometric-generative}, with further 
details provided in \cref{apx:Performance}, \ourmethod is in many ways \emph{orthogonal} to performance. This leads us to an interesting question: 
Are there other properties 
of a model's latent space, also orthogonal to performance, that can describe a the model's properties as 
a generator? 
While this merits its own extensive study, 
for the purposes of this work, 
we are interested in designing 
an experiment that could relate a notion of \emph{unsupervised} generative similarity to \ourmethod.
Using the experimental design described above, namely $\subspaces$, we suggest generating \emph{comparable} images that 
were unseen during training using the following pipeline. For each $\subspace \in \subspaces$: 
\begin{inparaenum}[(1)]
	\item Compute the centroid $C$ of the original latent space $E$.
	\item For $v\in \subspace$, compute $\tilde{v} = (1-t)v + tC$, where $t\in \reals$ is a (small) perturbation parameter. 
	\item Use the decoder associated with $\probe$ to generate a new image $\genimage$.
\end{inparaenum} 

This establishes a set of generated images $\genimages \coloneq \{\genimage : v \in \subspace \}$ that is robust to multiverse considerations---as various latent spaces are encoded into very different scales and coordinate systems, directly sampling 
from the latent space becomes impractical. 
In contrast, we provide an unsupervised approach for generating 
principled comparisons between $\subspace,\probe'(S)$ in the pixel space by comparing the batch Mean-Squared 
Error (MSE) between $\genimages$ and $\genimages'$. We display our results for measuring the correlation 
between \ourmethod and the \emph{generative distance matrices} over different random draws of $\subspaces$,
such that each cell in the matrix $M[i,j]$ that compares $E_i$ to $E_j$ is computed by $MSE(\genimages_i,\genimages_j)$, in the left panel of \Cref{fig:generative-cosine}.

\subsubsection{\ourmethod's Relation to Performance}\label{apx:Performance}
Continuing the corresponding discussion in the main paper, 
we now further analyze the relationship between \ourmethod,
an inherently structural measure, and performance. \Cref{fig:geometric-generative} from the main text
highlights that, in many of our evaluations, the latent spaces 
from the most stable and reliable \vae architecture (\bvae) show no correlation 
between \ourmethod and performance. In a similar vein, 
\Cref{fig:performance-wae-info} depicts the relationship between landscape norms and performance (MSE reconstruction loss on the test set) for \wae and \info. 
p vs. landscape norm.
Like \bvae, landscape norms for \wae appear to be orthogonal to performance, 
whereas the conclusion for \info (our overall weakest performer) is not nearly as clear and merits additional exploration in future work. 
Across all model architectures, however, 
\ourmethod demonstrates the capacity to characterize differences between the latent spaces of models with \emph{similar performance}.

\begin{figure}[H]
	\centering
	\includegraphics[width=\linewidth]{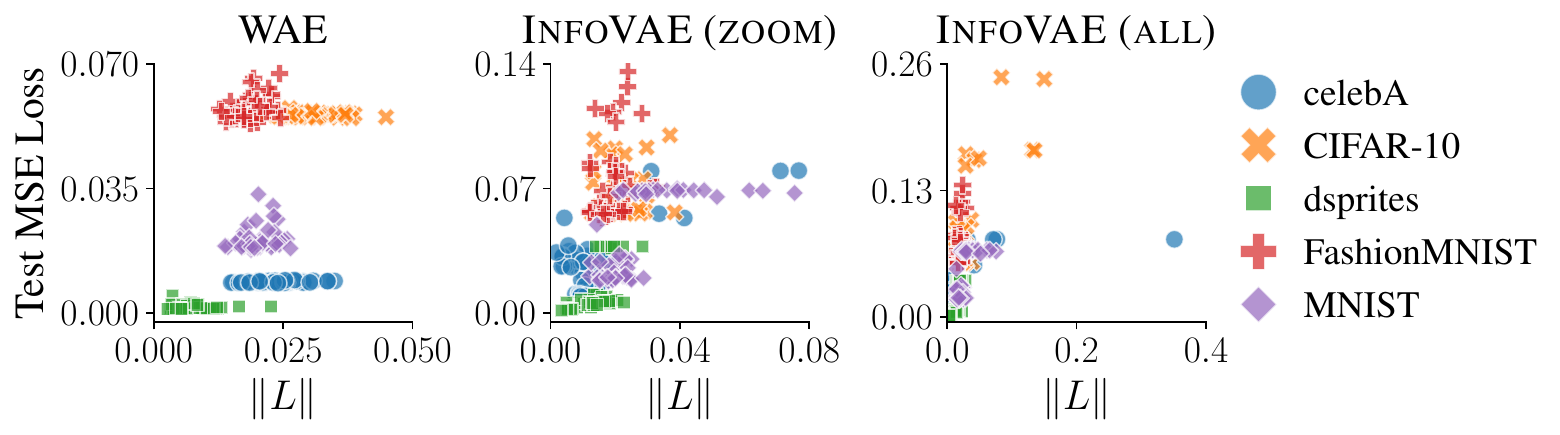}
	\caption{\textbf{Relationships between landscape norms and performance}.
	To assess the relationship between generative-model performance and \ourmethod, 
	we compare landscape 
	norms, a proxy for topological 
	complexity, 
	and test reconstruction loss (MSE errors) for the \wae and \info architectures.
	We observe that \ourmethod can distinguish \wae models that perform
	similarly across multiple datasets, while \info's hyperparameter 
	space, which is sensitive with respect to performance and representational variability, results in 
	unnecessary complexity that should encourage additional care 
	when applying this model to known and unknown tasks.}\label{fig:performance-wae-info}
\end{figure}

\subsubsection{Low-Complexity Training}
\label{apx:exp:low-complexity-training}
Finally, we examine \ourmethod's ability to perform hyperparameter compression
in \emph{low-complexity} training environments. While we already established
exciting opportunities for hyperparameter reuse across datasets using \ourmethod
(cf. \Cref{fig:reusing-hyperparameters}), low-complexity training constitutes another avenue for leveraging  
\ourmethod's hyperparameter compression to dismantle the environmentally costly 
culture of brute-force performance-based hyperparameter
optimization. In \Cref{fig:low-complexity-training}, we investigate the stability of our compression routine 
when halving the size of the training set (i.e., reducing it to $50\%$ of its original size). 
\begin{figure}[h]
	\centering
	\includegraphics[width=0.75\linewidth]{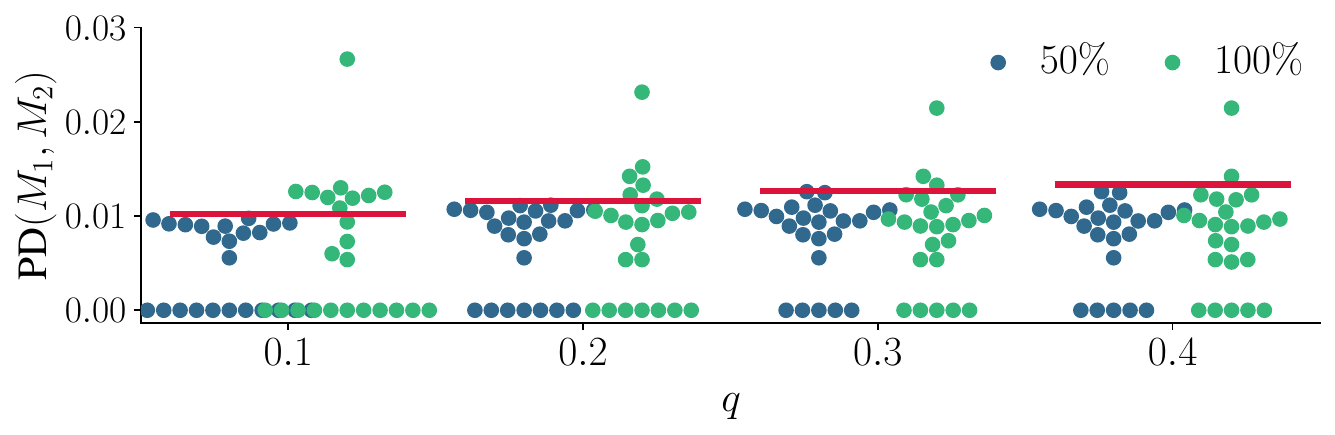}
	\caption{\textbf{Opportunities for low-complexity training.} 
		We show the pairwise \ourmethod distances between individual universes and their closest representatives in the \bvae \fashionmnist hyperparameter multiverse, 
		as assessed based on training with $50\%$ (blue) or $100\%$ (teal) of the training data, 
		where the set of representatives is computed based on the $50\%$-multiverse. 
		Red lines show the threshold value associated with the x-axis quantile $q$ in the $50\%$-multiverse, 
		and markers with $y=0$ indicate self-representation. 
		With \ourmethod, we can compress the hyperparameter search space in a low-complexity training setting and perform high-complexity training only on a smaller set of representatives, with limited topological distortion.\vspace*{-2cm}
	}\label{fig:low-complexity-training}
\end{figure} \clearpage
\subsection{Extended Experiments in Transformer Multiverses}
\label{apx:transformer-multiverse}
\label{apx:pipeline-alternatives}
\label{apx:presto-pipeline-multiverse}

Here, we provide further details on the configuration of our transformer multiverse, 
as well as additional experiments assessing the effect of choices in the \ourmethod pipeline on our landscape norms, distances, and sensitivity scores. 

\subsubsection{The Transformer Multiverse}
\label{apx:presto-transformers-description}

While our VAE multiverses focused on representational variability induced by hyperparameter and implementation choices, 
our transformer multiverse is designed to investigate both \ourmethod's power as a black-box diagnostic 
and the impact of choices in the \ourmethod pipeline on our measurements. 
To this end, 
we embed $2^{14} = 16\,384$ abstracts from four summarization datasets using six pretrained transformer models, 
covering different levels of textual technicality and language-model sophistication. 
Thus, we obtain $24$ sets of ($2^{14} \times \ast$)-dimensional embeddings, 
where $\ast \in \{384, 768, 1\,024, 1\,536\}$ is the embedding dimensionality of the respective language model. 
From these embeddings, we generate projections onto $\projdim\in[4]$ components via PCA or Gaussian random projections, 
using $\nsamples\in\{2^i\mid i \in \{10,\ldots,14\}\}$ embedded samples to compute persistent homology, 
as well as $\nprojections\in\{2^i\mid i\in \{3,\ldots,9\}\}$ projections to average landscapes when using random projections. 
We summarize the setup of our transformer multiverse in \cref{tab:transformer-multiverse}. 

\begin{table}[h]
	\centering
	\begin{tabular}{lr}
		\toprule
		Decision $\params_i$&Values\\
		\midrule
		Model \model&\ada, \distilroberta, \MiniLM, \mistral, \mpnet, \qadistilbert\\
		Dataset \dataset&\arxiv, \bbc, \cnn, \patents\\
		Dataset sample size \nsamples&$2^i$ for $i\in\{10,11,12,13,14\}$\\
		Number of projection components \projdim&$1,2,3,4$\\
		Projection method&PCA, Gaussian Random Projections\\
		Number of Gaussian projections \nprojections&$2^i$ for $i\in\{3,4,5,6,7,8,9\}$\\
		\bottomrule
	\end{tabular}
	\caption{\textbf{Transformer Multiverse.} 
		We work with $6\cdot 4 = 24$ sets of embeddings, 
		investigating each through the lens of multiple combinations of different choices involved in the \ourmethod pipeline. 
	}
	\label{tab:transformer-multiverse}
\end{table}

All datasets, as well as our four smaller transformer models, are available on HuggingFace. 
\begin{compactenum}[(1)]
	\item \textbf{Datasets.}
	\begin{compactenum}
		\item \href{https://huggingface.co/datasets/gfissore/arxiv-abstracts-2021}{\arxiv}: abstracts of all arXiv articles up to the end of $2021$;
		\item \href{https://huggingface.co/datasets/EdinburghNLP/xsum}{\bbc}: summaries of BBC news articles; 
		\item \href{https://huggingface.co/datasets/cnn_dailymail}{\cnn}: summaries of news articles from CNN and DailyMail; and
		\item \href{https://huggingface.co/datasets/big_patent}{\patents}: abstracts of U.S. patent applications.
	\end{compactenum}
	We embed the first $2^{14}$ samples from the designated \emph{training} sets of these datasets with each of our models. 
	\item \textbf{Models.} 
	\begin{compactenum}
		\item \href{https://huggingface.co/sentence-transformers/all-distilroberta-v1}{\distilroberta}: general-purpose model, embedding dimension $768$, maximum sequence length $512$ word pieces;
		\item \href{https://huggingface.co/sentence-transformers/all-MiniLM-L6-v2}{\MiniLM}: general-purpose model, embedding dimension $384$, maximum sequence length $256$ word pieces;  
		\item \href{https://huggingface.co/sentence-transformers/all-mpnet-base-v2}{\mpnet}: general-purpose model, embedding dimension $768$, maximum sequence length $384$ word pieces; and
		\item \href{https://huggingface.co/sentence-transformers/multi-qa-distilbert-cos-v1}{\qadistilbert}: QA-specialized model, embedding dimension $768$, maximum sequence length $512$ word pieces.
	\end{compactenum}
	These models are provided as pretrained by the \textit{sentence-transformers} library \cite{reimers2019sentencebert}. 
	They come with \emph{normalized} embeddings, 
	and we \emph{truncate} the texts to be embedded in the (rare) event that they exceed a model's maximum sequence length. 
\end{compactenum} 
To obtain embeddings from our two large language models, \ada (embedding dimension: $1\,536$) and \mistral (embedding dimension: $1\,024$), we query the corresponding APIs of their providers (OpenAI and MistralAI, respectively). 

\subsubsection{Landscape Norms}
\label{apx:transformers:norms}

To start, we investigate how our choice of \emph{projector} interacts with our chosen \emph{sample size} \nsamples, 
keeping the number of projection components fixed at $\projdim = 3$. 
In particular, 
we are interested in how our landscape norms change as we vary these parameters,  
since landscape norms underlie both our \ourmethod \emph{distances} and our \ourmethod \emph{sensitivities}. 
Inspecting the distributions of landscape norms when working with Gaussian random projections in \Cref{fig:norms-gauss}, 
we see that norms are approximately normally distributed, 
and that \emph{larger sample sizes} are associated with 
\emph{smaller landscape-norm means} as well as \emph{smaller landscape-norm variance}. 

\begin{figure}[t]
	\centering
	\includegraphics[width=\linewidth]{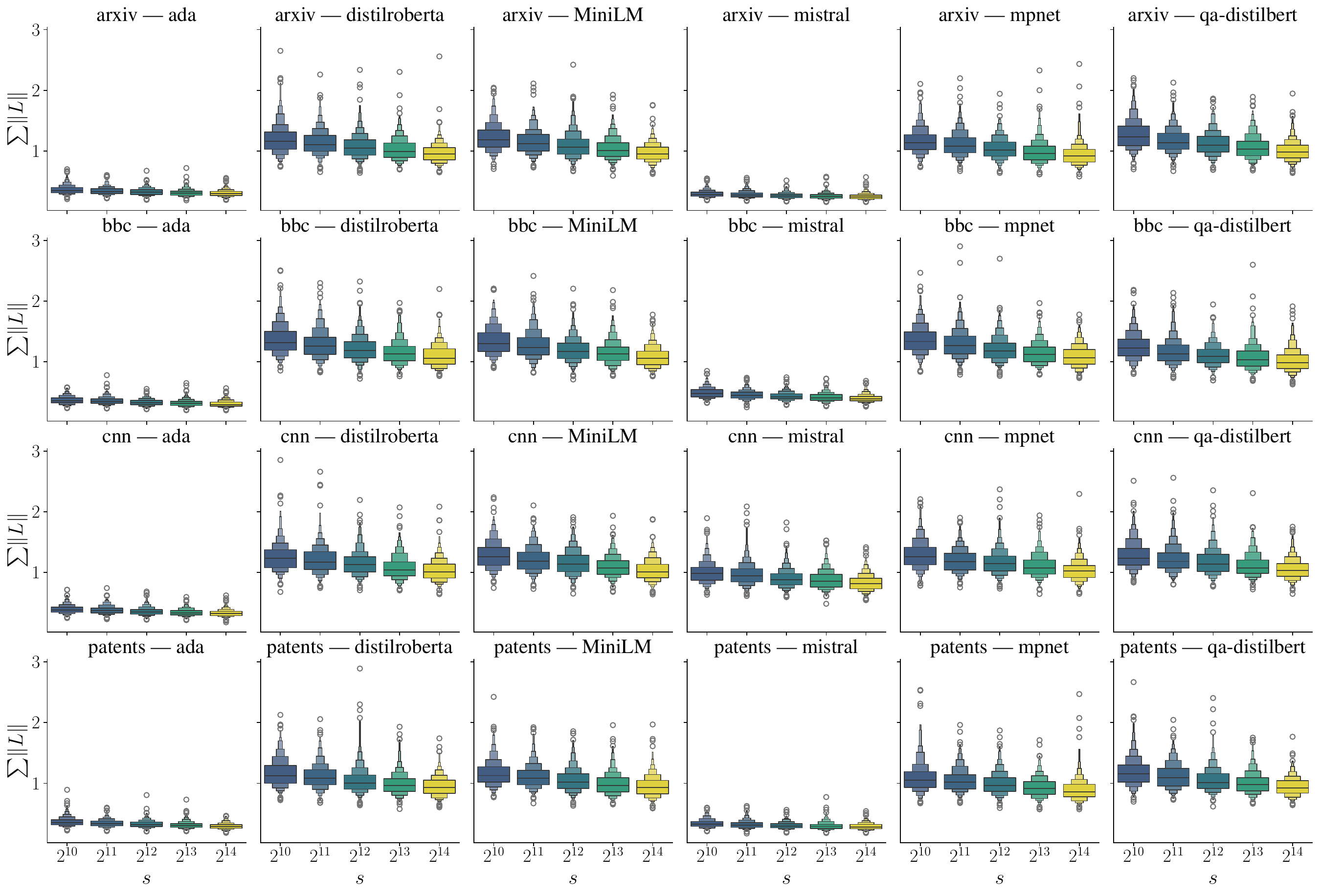}
	\caption{\textbf{Distribution of landscape norms when projecting via Gaussian random projections and varying the sample size \nsamples.} 
		We show the distribution of landscape norms when working with $\nprojections = 2^9 = 512$ Gaussian random projections and varying the sample size \nsamples of the embedded dataset. 
		Larger sample sizes are associated with smaller landscape norms, and the landscape norm distributions of large language models are clearly distinguishable from those of smaller transformer models.}
	\label{fig:norms-gauss}
\end{figure}
\begin{wrapfigure}{r}{0.5\linewidth}
	\centering
	\includegraphics[width=\linewidth]{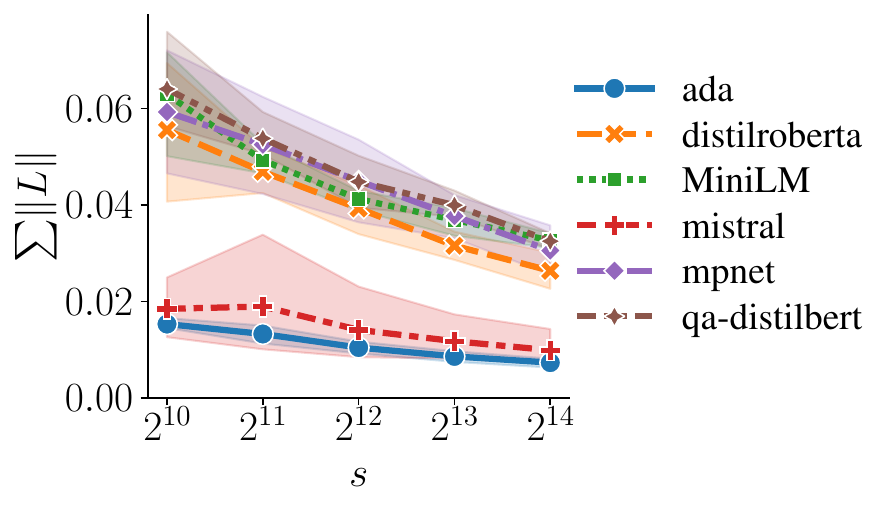}
	\caption{\textbf{Landscape norms when projecting with PCA and varying the sample size \nsamples.} 
		We show the landscape norms of PCA-based projections as a function of the sample size \nsamples of the embedded dataset. 
		While landscape norms obtained via PCA live on a scale which is different from that of landscape norms obtained via Gaussian Random projections, their distinguishing power and behavior under sample-size variation are qualitatively similar.}
	\label{fig:norms-pca}
\end{wrapfigure}

Next, we compare the distributions of norms under Gaussian random projections with the fixed norms we obtain under PCA.
We observe that although the norms live on different scales, 
with PCA-based norms being much smaller than the norms of Gaussian random projections, 
the change pattern when varying the sample size \nsamples is qualitatively similar, 
and the distinction between large language models and smaller transformer models is equally evident. 
Hence, while Gaussian random projections allow us explore representational variability \emph{within} individual latent spaces, 
when comparing \emph{between} latent spaces, 
we may use PCA as well. 
Thus, we fix PCA as a projector for the remainder of our experiments.

\clearpage

\begin{figure}[t]
	\centering
	\includegraphics[width=\linewidth]{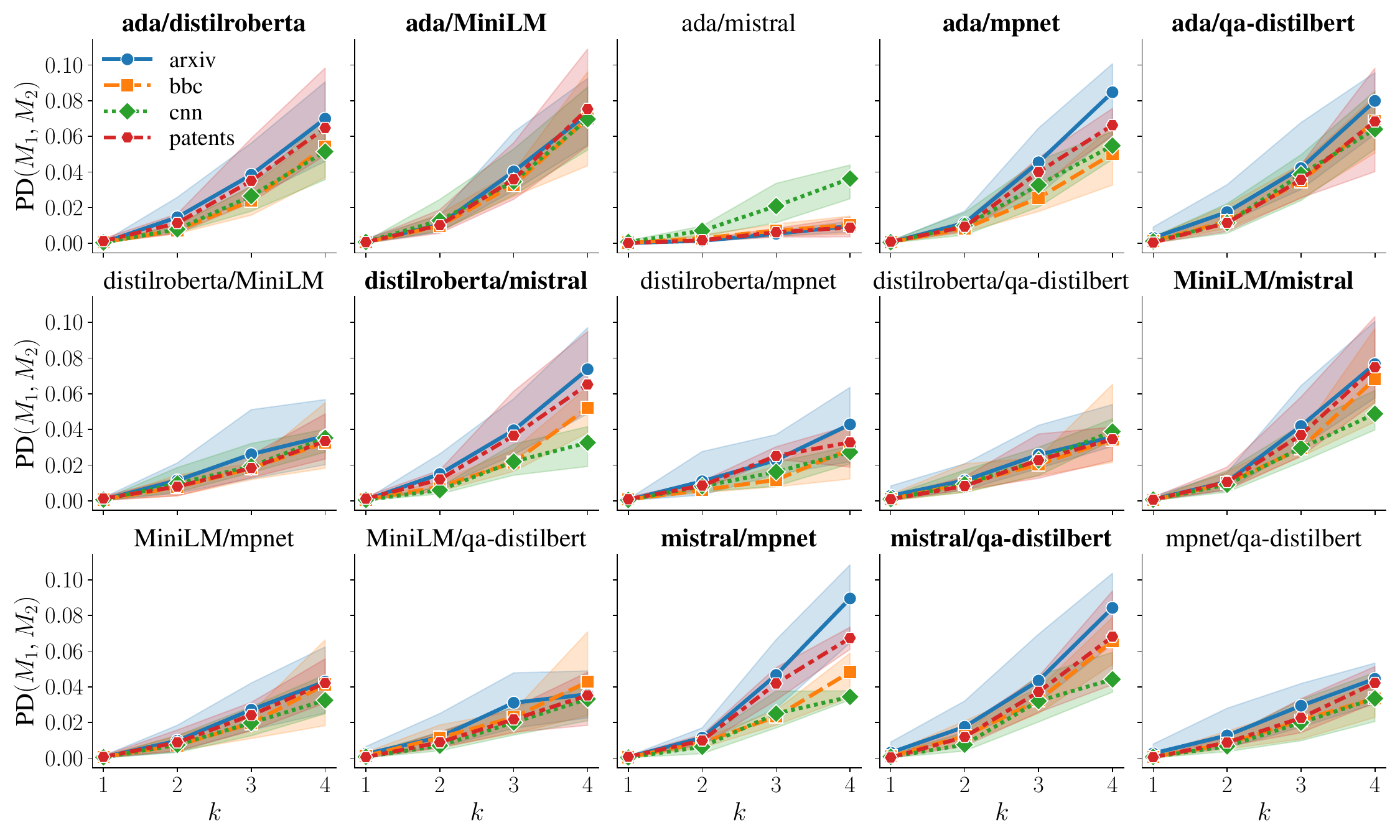}
	\caption{\textbf{\ourmethod distances between transformer models.}
		We show the \ourmethod distance between transformer models as measured on different datasets (lines) as a function of the number of projection components \projdim, 
		where the bands indicate the range of measurements across varying dataset sample sizes \nsamples (from $2^{10}$ to $2^{14}$), 
		and we typeset comparisons involving exactly one large language model in bold. 
		For $\projdim > 2$ projection components, \ourmethod distances distinguish large language models from smaller models based on technical datasets.
	}\label{fig:language-pairwise-nc}
\end{figure}

\subsubsection{\ourmethod Distances}
\label{apx:transformers:distances}

In \Cref{fig:transformer-mms-compare}, 
we showed the pairwise distances between the multiverse metric spaces associated with our transformer models. 
For a more fine-grained perspective, 
in \Cref{fig:language-pairwise-nc}, 
we additionally depict the range of pairwise \ourmethod distances between individual models as we vary the number of projection components \projdim and the number of samples \nsamples considered. 
We see that as we increase the number of projection components, \ourmethod's capacity to distinguish large language models from smaller transformer models increases. 
Furthermore, in many comparisons involving exactly one large language model, 
for $\projdim > 2$, 
the \ourmethod distance between the compared models is larger for our technical datasets (\arxiv, \patents) than for our news datasets (\bbc, \cnn).

\subsubsection{\ourmethod Sensitivities}
\label{apx:transformers:sensitivities}

Complementing our assessment of \ourmethod sensitivities in the VAE hyperparameter and implementation multiverses (\Cref{fig:vae-sensitivity,fig:vae-implementation-sensitivity}), 
we now investigate the local \ourmethod sensitivities of models and datasets in our transformer multiverse 
as a function of the number of samples \nsamples and the number of projection components \projdim. 
When exploring sensitivity to \emph{model variation} in \Cref{fig:transformer-sensitivity-models}, 
we observe two trends: 
\ourmethod sensitivities become \emph{smaller} as we increase the number of samples \nsamples, 
and they become \emph{larger} as we increase the number of projection components \projdim. 
These trends persist when we turn to \emph{dataset variation}, depicted in \Cref{fig:transformer-sensitivity-datasets}. 
Since a smaller number of samples provides a rougher picture of a model's latent space, 
and a larger number of projection components allows us to capture more variation, 
these patterns are to be expected, 
which further increases our confidence in \ourmethod.
Finally, \Cref{fig:transformer-sensitivity-datasets}  
reveals that the \ourmethod sensitivity when varying \emph{datasets} differs widely between \emph{models}---and notably, \ourmethod's dataset sensitivity separates our two large language models, \ada and \mistral (cf. \Cref{fig:transformer-sensitivity-datasets}, middle panel). 

\clearpage

\begin{figure}[h]
	\centering
	\includegraphics[width=0.66\linewidth]{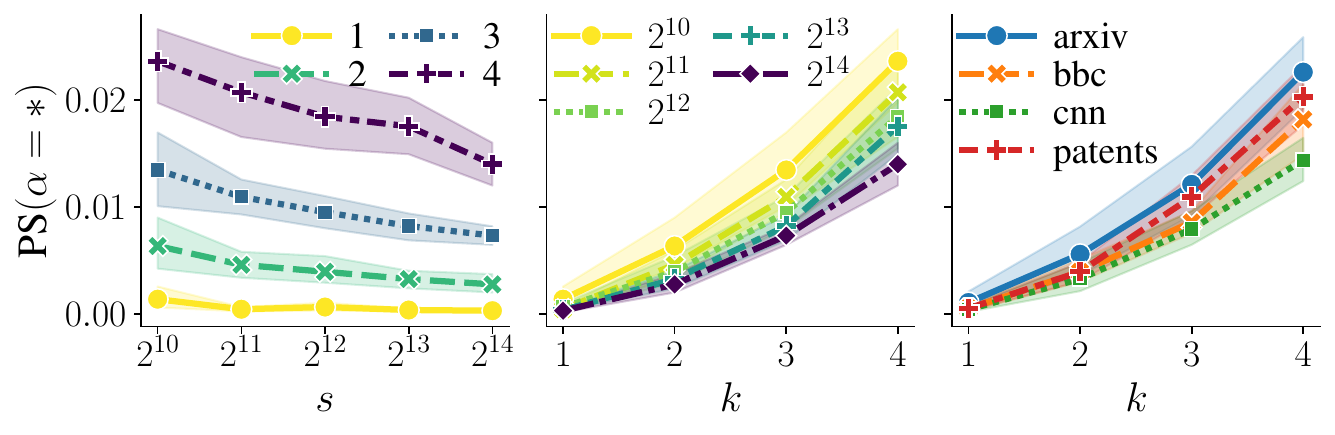}
	\caption{\textbf{Local \ourmethod sensitivity of \emph{models} in the transformer multiverse.}
		In the left and middle panels, we show the local \ourmethod sensitivity scores of the model choice \algo (lines), 
		along with the range of individual \ourmethod sensitivities across our \emph{models} (bands), 
		as a function of the number of samples \nsamples and the number of projection components \projdim.
		In the right panel, 
		we show the local \ourmethod sensitivity scores of the model choice \algo (lines), 
		along with the range of individual \ourmethod sensitivities across \emph{sample sizes} \nsamples (bands),
		as a function of the number of projection components \projdim and the dataset \data. 
		The smaller the number of samples, and the larger the number of projection components, the higher the relevance of the model choice.
	}\label{fig:transformer-sensitivity-models}
\end{figure}

\begin{figure}[h]
	\centering
	\includegraphics[width=0.66\linewidth]{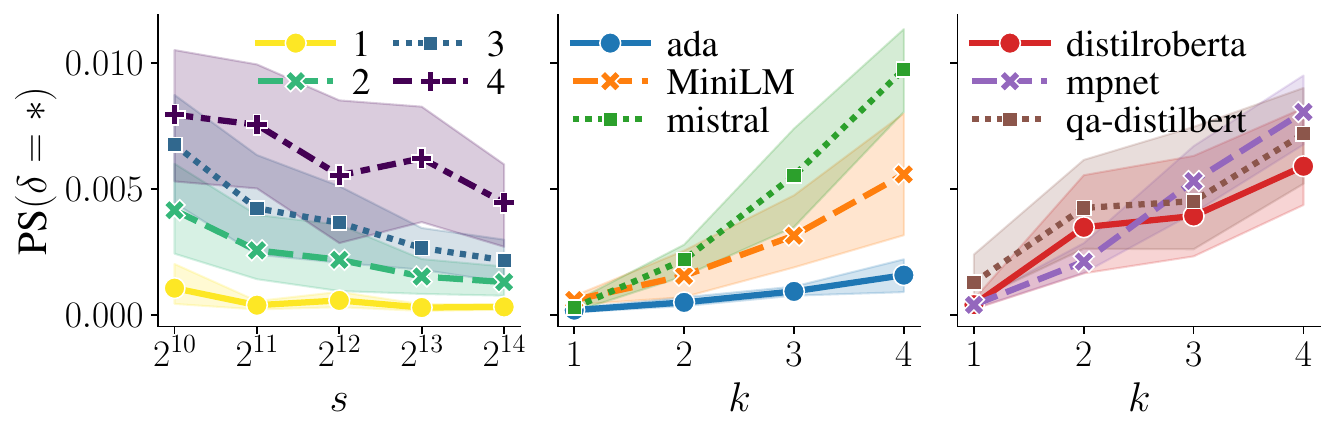}
	\caption{\textbf{Local \ourmethod sensitivity of \emph{datasets} in the transformer multiverse.}
		In the left panel, we show the local \ourmethod sensitivity scores of the dataset choice \data (lines), 
		along with the range of individual \ourmethod sensitivities across our \emph{datasets} (bands), 
		as a function of the number of samples \nsamples and the number of projection components \projdim.
		In the middle and right panels, 
		we show the local \ourmethod sensitivity scores of the dataset choice \data (lines), 
		along with the range of individual \ourmethod sensitivities across our \emph{number of samples} \nsamples (bands),
		as a function of the number of projection components \projdim and the transformer model \algo.
		Both the local \ourmethod sensitivity and the variance in individual \ourmethod sensitivities are highly dependent on the chosen transformer model.
	}\label{fig:transformer-sensitivity-datasets}
\end{figure}

 \subsection{\ourmethod as a Dissimilarity Measure for Neural-Network Representations}
\label{apx:neural-networks}

\ourmethod is designed for comparisons between general latent spaces, 
not neural representations specifically, 
and we deliberately deferred an in-depth treatment of \ourmethod for neural-network forensics to future work. 
To gauge the potential of \ourmethod as a dissimilarity measure in this context, 
we compare \ourmethod with two variants of CKA \citep{kornblith2019similarity}. 
When measured by CKA, representations in neighboring neural-network layers are consistently judged as more similar to each other than to representations in more distant neural-network layers. 
Our preliminary experiments using the MNIST-trained neural-network representations made publicly available in the \href{https://github.com/yuanli2333/CKA-Centered-Kernel-Alignment/tree/master/model_activations/MNIST}{CKA repository}, 
summarized in \cref{tab:presto-vs-lkcka},
indicate that \ourmethod captures this trend as well. 
In \Cref{fig:nnsim}, 
we additionally depict the similarity ranks of \ourmethod (transformed into a similarity measure) as well as $l$CKA and $k$CKA 
for relationships between different layers in the same model (\emph{intra-model relationships}) 
as well as relationships between the same or different layers of differently seeded models (\emph{inter-model relationships}).
However, we also emphasize that there is no reason to expect, a priori, that the relationships between the internal representations of a neural-network model should be quantitatively the same across differently (hyper)parameterized neural-network models. 
Rather, this is an interesting hypothesis that \ourmethod will allow us to test going forward.

\begin{table}[t]
	\centering
	\begin{tabular}{rrrr}
		\toprule
		&\multicolumn{3}{c}{Type of Correlation}\\
		\cmidrule{2-4}
		Comparison&Pearson&Spearman&Kendall\\\midrule
		\ourmethod vs. $l$CKA&$-0.88$ ($p = 0.02$)&$-0.89$ ($p = 0.02$)&$-0.73$ ($p = 0.06$)\\
		\ourmethod vs. $k$CKA&$-0.75$ ($p = 0.09$)&$-0.89$ ($p = 0.02$)&$-0.73$ ($p = 0.06$)\\
		\bottomrule
	\end{tabular}
	\caption{\textbf{\ourmethod's relationship to CKA variants.} 
		We display the correlations between \ourmethod (a dissimilarity measure) and $l$CKA resp. $k$CKA (similarity measures) for neural-network models trained on MNIST using different seeds. 
		}\vspace*{-6pt}
	\label{tab:presto-vs-lkcka}
\end{table}

\begin{figure}[t]
	\centering
	\begin{subfigure}{0.275\linewidth}
		\centering
		\includegraphics[width=0.3\linewidth]{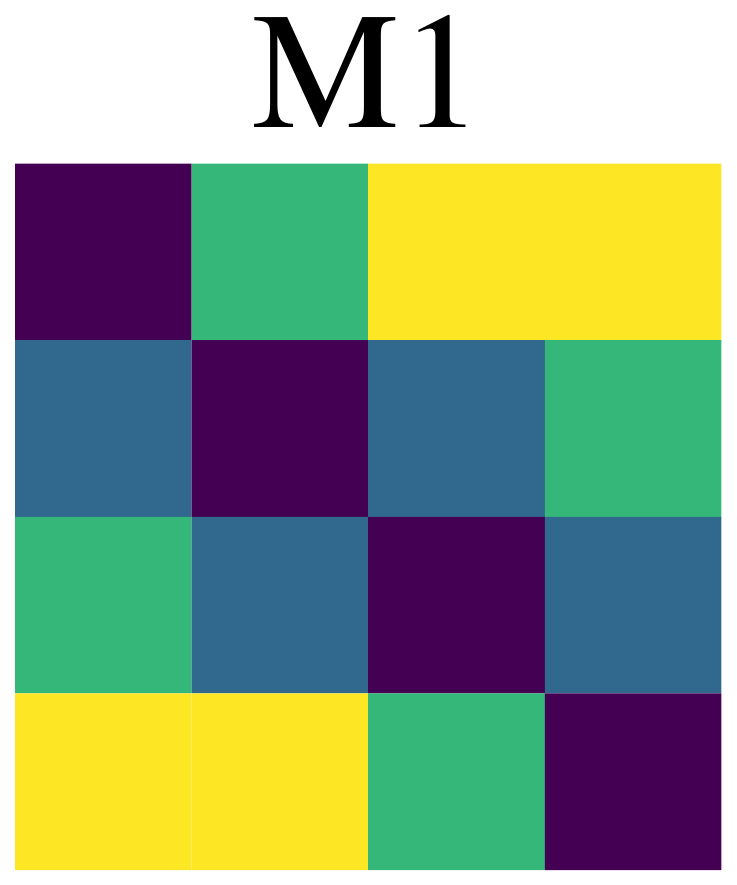}~\includegraphics[width=0.3\linewidth]{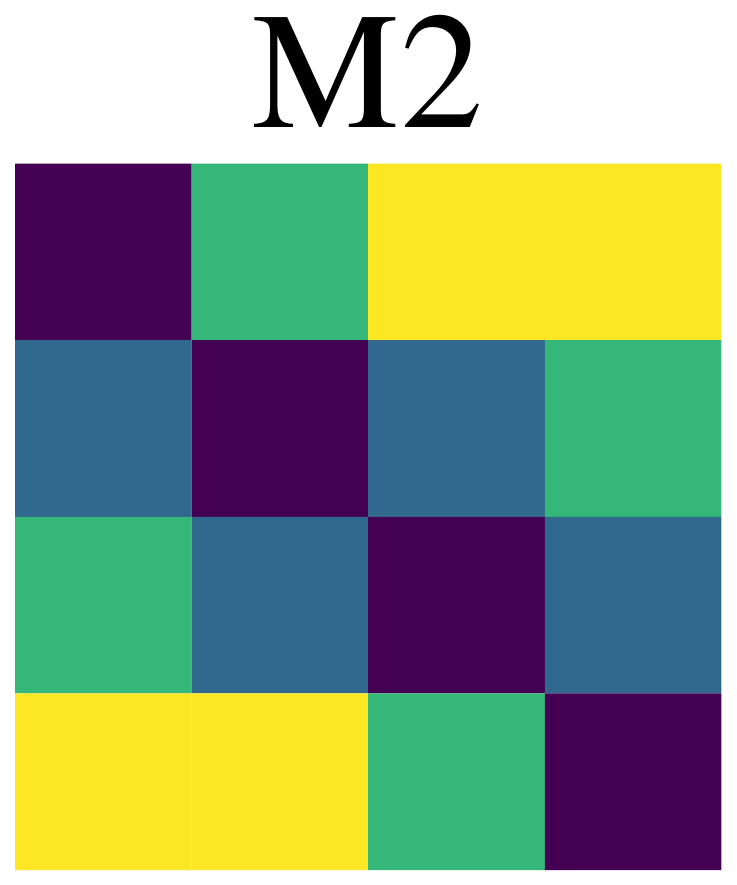}~\includegraphics[width=0.3\linewidth]{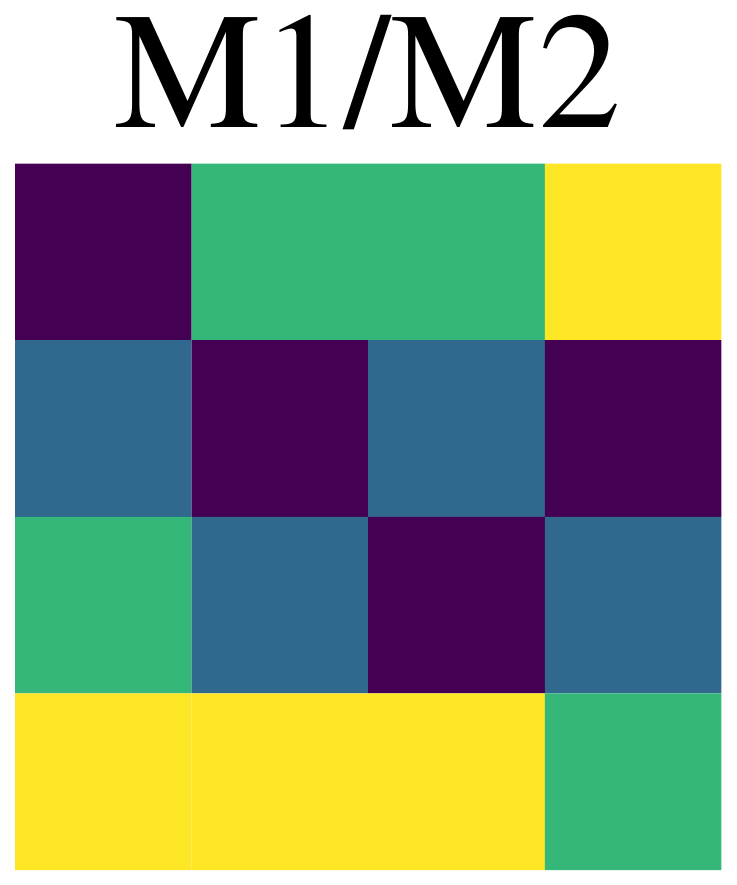}
		\subcaption{$l$CKA}
	\end{subfigure}~\begin{subfigure}{0.275\linewidth}
		\centering
		\includegraphics[width=0.3\linewidth]{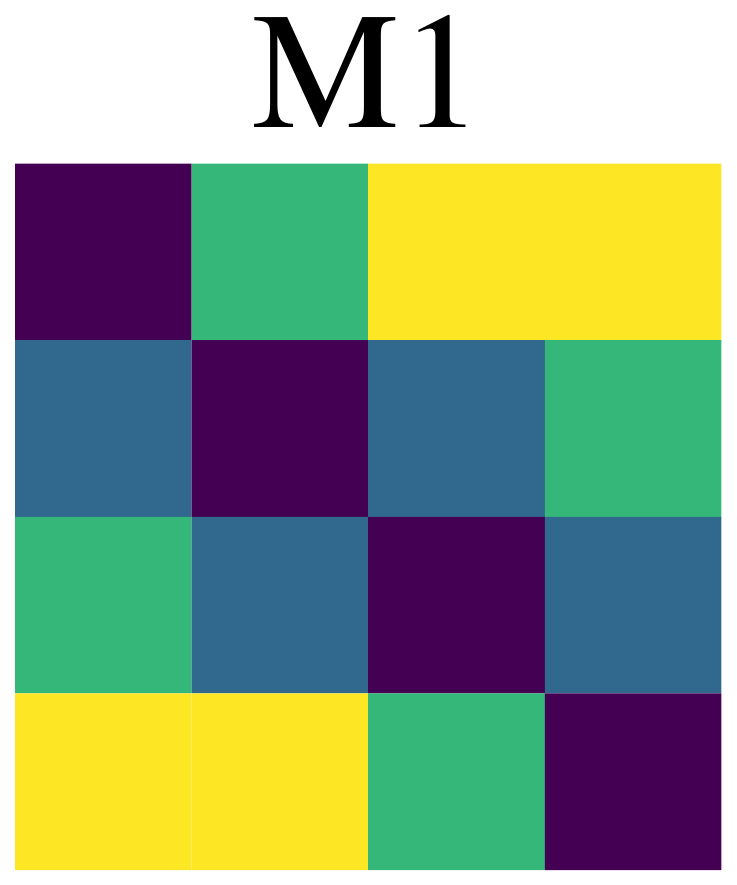}~\includegraphics[width=0.3\linewidth]{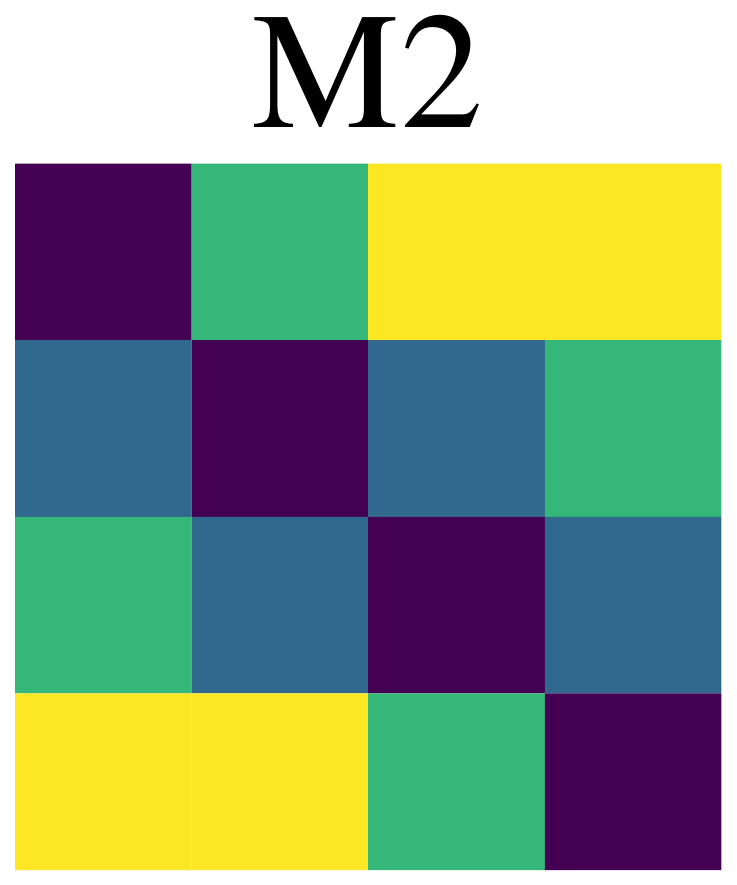}~\includegraphics[width=0.3\linewidth]{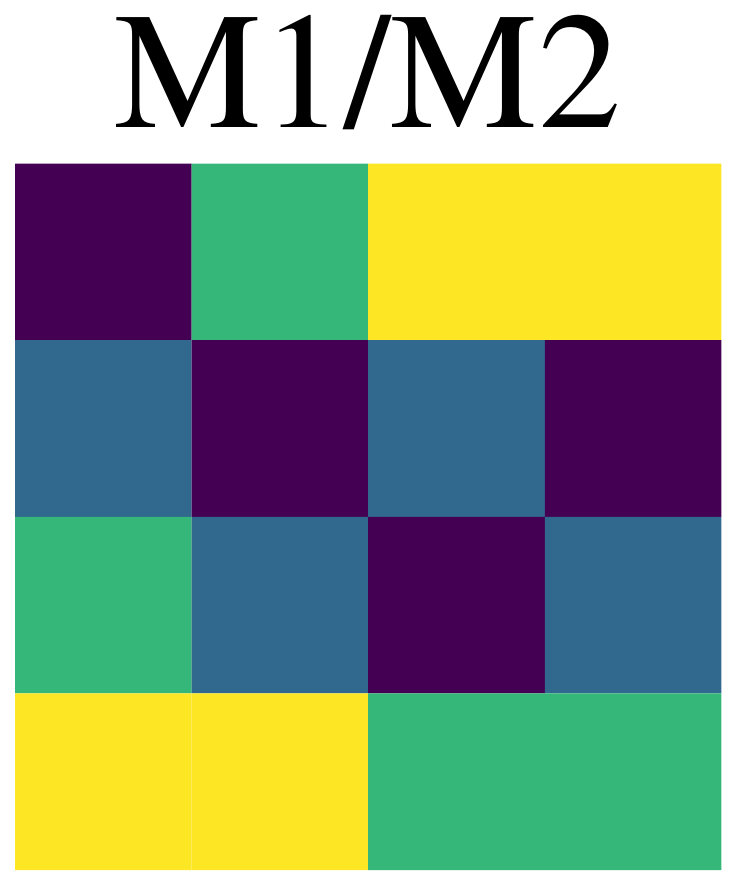}
		\subcaption{$k$CKA}
	\end{subfigure}~\begin{subfigure}{0.275\linewidth}
		\centering
		\includegraphics[width=0.3\linewidth]{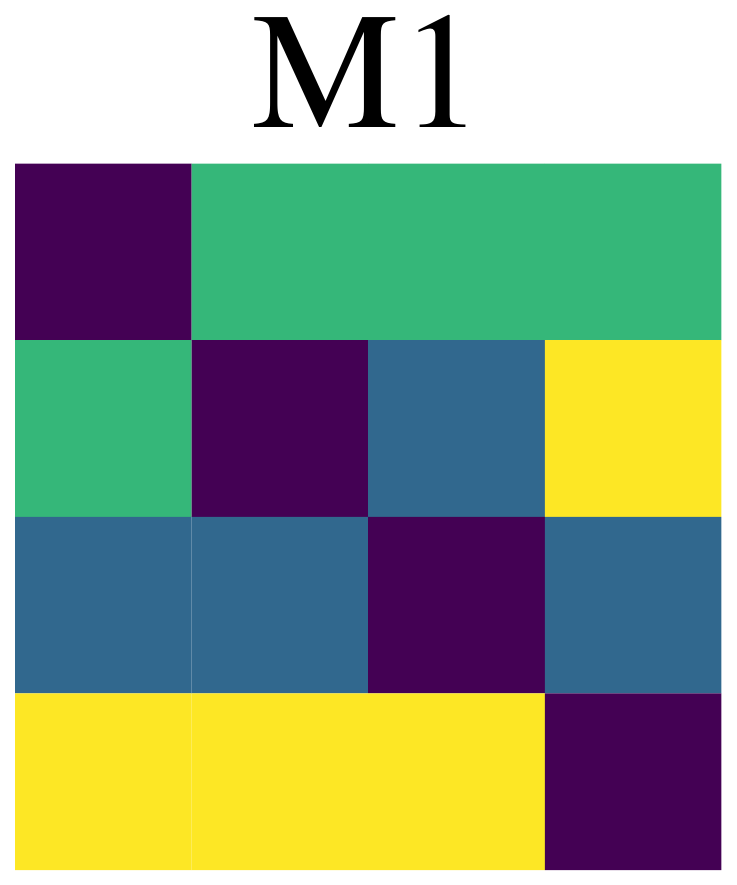}~\includegraphics[width=0.3\linewidth]{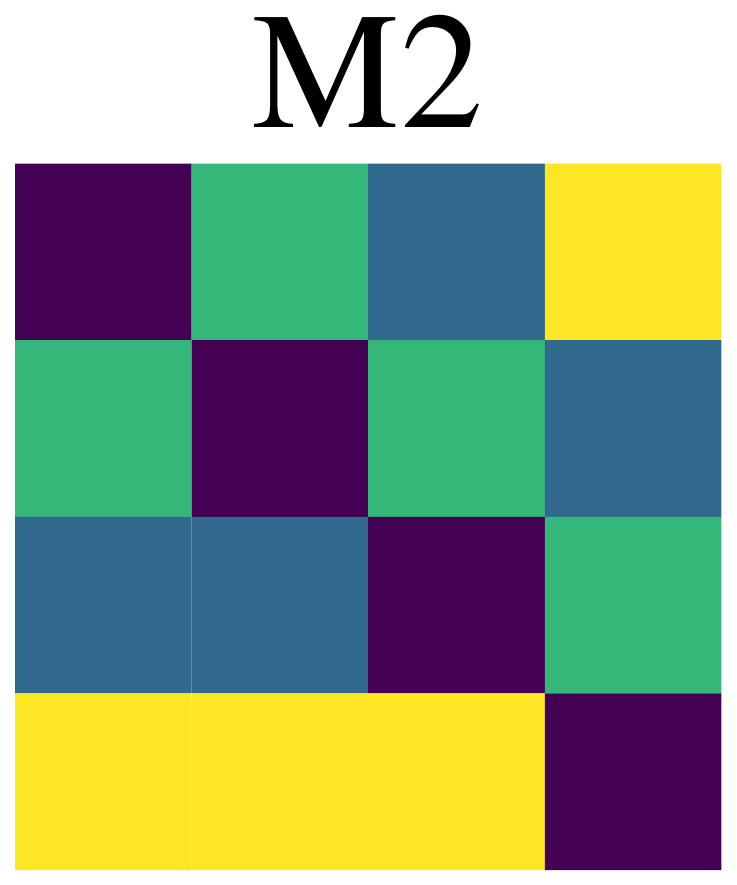}~\includegraphics[width=0.3\linewidth]{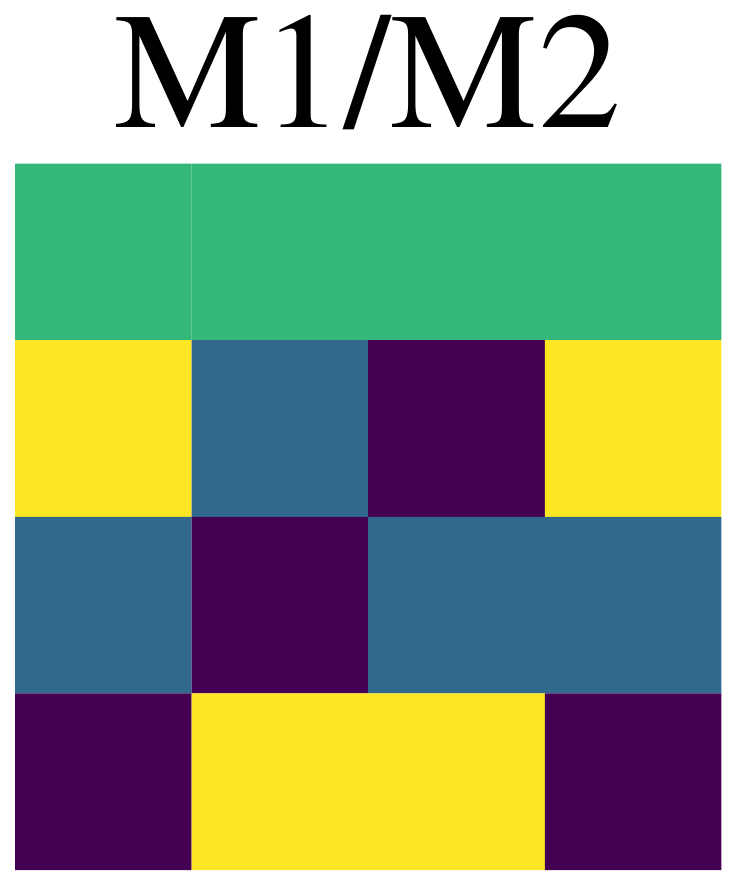}
		\subcaption{\ourmethod}
	\end{subfigure}
	\caption{\textbf{Comparing neural activations.}
		We replicate an experiment by \citet{kornblith2019similarity}, who assess the similarity of neural activations in $4$-layer CNNs trained on MNIST, 
		comparing \ourmethod (turned into a similarity measure) with similarities produced by $l$CKA and $k$CKA. 
		For each method, we show the similarity ranks for different layers in the same model (left and middle subpanels) as well as for different layers across two different models (right subpanels).
		Here, similarity ranks are computed column-wise, 
		and darker colors indicate higher ranks 
		(i.e., if in the comparison between model~1 [M1] and model~2 (M2), the square in position $[i,j]$ is colored in dark blue, 
		layer $j$ of M1 is most similar to layer $i$ of M2).
		} 	\label{fig:nnsim}
\end{figure}

\subsection{Multiverse Analysis of Non-Linear Dimensionality-Reduction Methods} \label{apx:DimRed}
 
Manifold-learning techniques are essential tools in various applied sciences, 
including computational biology and medicine, 
where algorithms such as \umap, \tsne,
and \phate \cite{mcinnes_umap_2020,maaten_visualizing_2008,moon_visualizing_2019} 
are commonly used to generate low-dimensional embeddings of complex datasets,
imbuing these computationally tractable representations with geometric and topological
structure learned from the data manifold. 

\begin{wraptable}{r}{9cm}
	\centering \small
	\begin{tabular}{lrrrrr}
\toprule
&\multicolumn{5}{c}{Model}\\
\cmidrule{2-6}
Dataset & Isomap & LLE & Phate & t-SNE & UMAP \\
\midrule
Breast Cancer & 13.41 & 4.59 & 1.52 & 0.21 & 0.34 \\
Diabetes & 0.22 & 51.19 & 0.37 & 0.06 & 0.32 \\
Digits & 0.35 & 0.94 & 0.24 & 0.24 & 1.14 \\
Iris & 0.46 & 9.01 & 0.52 & 0.74 & 1.28 \\
Moons & 0.00 & 0.00 & 1.67 & 1.34 & 0.22 \\
Swiss Roll & 1.56 & 0.26 & 1.28 & 0.42 & 0.60 \\
\bottomrule
\end{tabular}

 	\caption{\textbf{Sensitivity of dimensionality-reduction methods.}
		We show the \ourmethod sensitivity scores of the \textit{number-of-neighbors} parameter for five dimensionality-reduction algorithms on six datasets. 
		Popular dimensionality-reduction vary widely in hyperparameter sensitivity.
	} \label{tab:n-neighbors-sensitivity}
\end{wraptable}

These methods also exhibit representational 
variability that can be measured with \ourmethod, 
especially when varying their hyperparameters.
Among the most critical parameters determining the structure of a low-dimensional representation is 
the parameter controlling the \emph{locality} of the dimensionality-reduction method, 
which manifests in various variants of a \emph{number-of-neighbors} parameter (named differently for different algorithms). 
In \cref{tab:n-neighbors-sensitivity}, 
we use \ourmethod's sensitivity scores based on the distribution of landscapes that arise from 
varying this locality parameter across different synthetic and real-world datasets. 

To further demonstrate \ourmethod's utility in the manifold-learning space,
in \Cref{fig:dim-red-clustering}, 
we cluster the latent representations arising from different combinations of dimensionality-reduction algorithms, 
hyperparameters, and datasets 
based on the multiverse metric space constructed from pairwise \ourmethod distances between embeddings. 
We are confident that \ourmethod-based 
tools will be useful for practitioners in the applied sciences by
\begin{inparaenum}[(1)]
    \item allowing them to assess the \emph{sensitivity} of their algorithms and datasets with respect to their hyperparameter choices, and
    \item helping them condense the multiverse of representations into a manageable set of structurally distinct representatives via hyperparameter \emph{compression}. 
\end{inparaenum}

\begin{figure}[H]
	\centering
	\includegraphics[height=2.7cm]{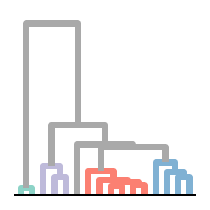}~\includegraphics[height=2.7cm]{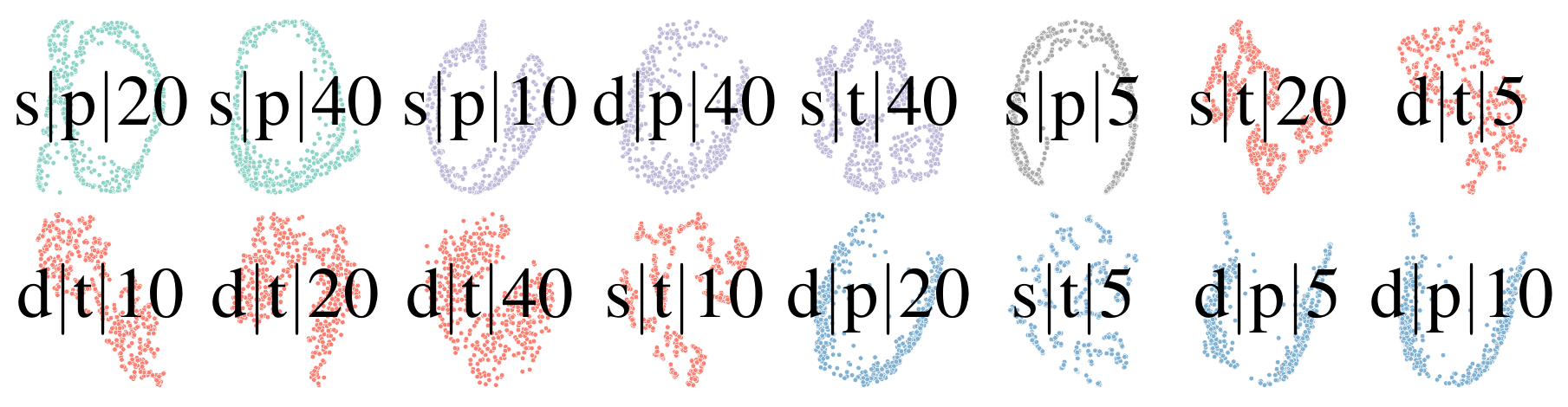}
	\caption{\textbf{Clustering embeddings with \ourmethod.}
		We show the dendrogram of a complete-linkage clustering of $2$-dimensional embeddings based on \ourmethod distances (left), 
		as well as the clustered embeddings, 
		colored by their cluster when cutting the dendrogram at distance $1$ (right).
		Annotations are of shape dataset$|$method$|$$n$, 
		where dataset $\in \{$swiss roll (s), diabetes (d)$\}$, 
		method $\in \{$phate~(p), t-SNE (t)$\}$, 
		and $n\in\{5,10,20,40\}$ is the nearest-neighbors parameter of both methods. 
		With \ourmethod, we can compare potentially unaligned embeddings originating from different datasets and methods.
	} \label{fig:dim-red-clustering}\vspace*{-12pt}
\end{figure}

 \clearpage
\section{Extended Methods}
\label{apx:methods}

In this section, we provide details on the computational complexity of topological multiverse analysis and detail how we can pick good representatives for (hyperparameter-)search-space compression. 

\subsection{\ourmethod Complexity}
\label{apx:computational-complexity}

\subsubsection{Theoretical Analysis}
\label{apx:computational-complexity:theoretical}
The individual \ourmethod steps of the \ourmethod pipeline exhibit the following computational complexities. 
\begin{compactenum}[(S1)]
	\item \textbf{Embed data.} 
	The complexity of computing $\embedding = \model(\dataset) \in\reals^{\nsamples\times \ast}$ for $\dataset\in\reals^{\nsamples\times \ast}$ 
	depends only on \dataset and \model, 
	and hence, is independent of our specific design choices.
	\item \textbf{Project embeddings.}
	Given embedding ${\embedding\in\reals^{\nsamples\times \dimension}}$, 
	computing \projdim-dimensional PCA takes $\BO(\projdim\nsamples\dimension)$ time via truncated SVD, 
	while generating \nprojections random projections of \embedding onto $\projdim$ dimensions takes $\BO(\nprojections\projdim\nsamples\dimension)$ time.
	\item \textbf{Construct persistence diagrams.} 
	When constructing an \topodim-dimensional $\alpha$-complex from \nsamples points,  
	the bottleneck is computing the Delaunay triangulation, 
	which takes worst-case expected time $\BO(\nsamples^{\lceil\nicefrac{\topodim}{2}\rceil+1})$ but often runs in $\BO(\nsamples\log\nsamples)$ time in practice.
\item \textbf{Compute persistence landscapes.} 
	If done exactly, computing a persistence landscape from \nsamples birth-death pairs takes $\BO(\nsamples^2)$ time. 
	However, at the cost of a small perturbation in our persistence diagrams, 
	we can round birth and death times to lie on a grid of constant size, 
	such that we can obtain persistence landscapes in $\BO(\nsamples \log \nsamples)$ time \cite{bubenik2017persistence}.
	When working with random projections, 
	averaging \nprojections \emph{exact} persistence landscapes takes $\BO(\nsamples^2\nprojections\log\nprojections)$ time, 
	whereas averages of \emph{approximate} landscapes can be computed in $\BO(\nsamples\nprojections)$ time \cite{bubenik2017persistence}. 
\end{compactenum}
Hence, for a constant number of latent dimensions \dimension and (if applicable) projections \nprojections, 
the topology-based steps of the \ourmethod pipeline can be performed in $\tildeO(\nsamples)$ time, 
i.e., our computations are approximately linear in the number of samples in \dataset. 
This also holds for the computation of our \emph{\ourmethod primitives}, 
i.e., 
the \ourmethod \emph{distance} (\ourdistance) and the \ourmethod \emph{variance} (\ourvariance).  
Overall, we obtain a scalable toolkit for the topological analysis of representational variability in latent-space models.

\subsubsection{Empirical Analysis}
\label{apx:computational-complexity:empirical} 
We supplement our theoretical analysis with \ourmethod's empirical running times, 
detailed in \cref{tab:empirical-runningtimes}. 
Note that these running times are based on a single-CPU implementation. 
Various optimization and parallelization strategies for persistent-homology calculations and diameter approximations exist, 
and \ourmethod can benefit from them directly.

\begin{table}[h]
	\centering
	\begin{subtable}{0.475\linewidth}
		\centering
		\begin{tabular}{rrrr}
			\toprule
			&\multicolumn{3}{c}{$d$}\\\cmidrule(lr){2-4}
			$s$& $128$ & $256$ & $512$\\ \midrule
			$2^{12}$ & $0.30 \pm 0.03$ & $0.30 \pm 0.01$ & $0.34 \pm 0.04$  \\
			$2^{14}$ & $1.16 \pm 0.02$ & $1.20 \pm 0.04$ & $1.26 \pm 0.03$  \\
			$2^{16}$ & $5.36 \pm 0.28$ & $5.28 \pm 0.11$ & $7.26 \pm 0.88$  \\
			$2^{18}$ & $25.86 \pm 1.31$ & $28.43 \pm 2.04$ & $36.26 \pm 9.70$  \\
			\bottomrule
		\end{tabular}
		\subcaption{\ourmethod without Normalization}
	\end{subtable}~\begin{subtable}{0.475\linewidth}
		\centering
		\begin{tabular}{rrrr}
			\toprule
			&\multicolumn{3}{c}{$d$}\\\cmidrule(lr){2-4}
			$s$& $128$ & $256$ & $512$\\ \midrule
			$2^{12}$ & $0.59 \pm 0.04$ & $0.69 \pm 0.03$ & $0.92 \pm 0.02$ \\
			$2^{14}$ & $2.35 \pm 0.04$ & $2.84 \pm 0.06$ & $4.01 \pm 0.14$ \\
			$2^{16}$ & $10.45 \pm 0.34$ & $12.32 \pm 0.41$ & $38.20 \pm 0.95$ \\
			$2^{18}$ & $67.54 \pm 2.27$ & $154.47 \pm 36.00$ & $156.71 \pm 45.48$ \\
			\bottomrule
		\end{tabular}
		\subcaption{\ourmethod with Normalization}
	\end{subtable}
	\caption{\textbf{Empirical running times of \ourmethod.} We report the average running times (seconds) of computing \ourmethod distances across random embeddings of varying sizes  on a single CPU. We compute 10 different pairs of randomly seeded embeddings for each size $(s,d)$. We project the embeddings using PCA into 2 dimensions and fit our landscapes using Alpha Complexes for homology dimensions 0 and 1.}
	\label{tab:empirical-runningtimes}
\end{table}

\subsubsection{Runtime Comparison with Other Methods} 
\label{apx:computational-complexity:comparative}
Unlike other methods, \ourmethod can compare latent spaces with hundreds of thousands of samples embedded into large latent dimensions on commodity hardware. 
In the age of large LLM embeddings, our method allows users with limited computational resources to run \ourmethod quickly and without using paralyzing amounts of memory. 
In \cref{tab:comparative-runningtimes}, 
we compare the running times of computing \ourmethod distances with the running times of computing several other (dis)similarity measures.
Here, ``Pairwise'' refers to computing basic pairwise distances in the high-dimensional space, 
``VR'' refers to constructing a Vietoris-Rips complex based on these distances, 
and ``IMD'' (Intrinsic Multi-scale Distance) refers to the method by \citet{tsitulin2020shape}. 
We do not include running times for Representation Topological Divergence (RTD) \citep{trofimov2023learning} because it is designed specifically for GPUs. 
Furthermore, we note that \ourmethod's computational complexity in large spaces hinges on the computation of alpha complexes, persistence landscapes, and diameter approximations (when normalizing spaces). These have the potential to be accelerated and parallelized nicely \citep{chazal2014stochastic,chazal2015subsampling}. 
We leave the integrations and analyses of these optimizations to future work and will reflect updates in our open-source implementation.

\begin{table}[t]
	\centering
	\begin{tabular}{rrrrrrrr}
		\toprule
		$s$ & $d$ & \ourmethod (no norm.)& \ourmethod  (norm.) & \text{CKA} & \text{Pairwise} & \text{VR} & \text{IMD} \\
		\midrule
		$2^{12}$ & $128$ & $0.308$ & $0.721$ & $22.568$ & $0.23$ & $3.113$ & $2.692$ \\
		$2^{12}$ & $256$ & $0.302$ & $1.149$  & $24.643$ & $0.446$ & $5.558$ & $4.772$ \\
		$2^{12}$ & $512$ & $0.313$ & $2.184$  & $27.689$ & $0.692$ & $9.641$ & $6.633$ \\
		$2^{14}$ & $128$ & $1.236$ & $2.816$ & $2\,601.56$ & $9.614$ & $170.41$ & $33.968$ \\
		$2^{14}$ & $256$ & $1.197$ & $4.652$ & $2\,646.33$ & $10.649$ & $220.507$ & $64.061$ \\
		$2^{14}$ & $512$ & $1.248$ & $11.096$ & $2\,445.32$ & $18.081$ & $272.997$ & $101.336$ \\
		\bottomrule
	\end{tabular}
	\caption{\textbf{Comparative running times.} We compare \ourmethod to other representational similarity methods, reporting the run times (seconds) for \ourmethod (with and without normalization), CKA, Pairwise Distances (\texttt{scikit-learn}), Vietoris Rips (\texttt{guhdi}), and IMD on a pair of random embeddings of varying sizes on a single CPU. \ourmethod times are averaged over $10$ pairs of embeddings.}
	\label{tab:comparative-runningtimes}
\end{table}

\subsection{\ourmethod Compression}
\label{apx:compression}

When using \ourmethod for search-space compression, 
our goal is to select a small set of \emph{representatives} \representatives (e.g., hyperparameter vectors) to explore in detail 
such that each universe in the original search space has a representative in the compressed search space at topological distance no larger than $\epsilon$. 
To obtain an efficient set of such representatives 
(i.e., a small set of configurations that together satisfy our topologically-dense-sampling criterion), 
we can take two approaches. 
First, we can \emph{cluster} a multiverse based on the pairwise topological distances between its universes 
using a method that bounds intra-cluster distances
(e.g., using agglomerative complete-linkage clustering and cutting the dendrogram at height~$\epsilon$)
and pick one representative from each cluster. 
Alternatively, 
to bound the size of the representative set as a function of its minimum size, 
we can interpret topologically dense sampling as a \emph{set-cover problem} (or equivalently, a \emph{hitting-set problem}), 
where our universes are both the elements (to be represented themselves) and the candidate sets (representing themselves and others). 
While minimum set cover is NP-hard \cite{karp1972reducibility},  
the simple \emph{greedy approximation algorithm}  
that picks the candidate universe capable of representing the largest number of unrepresented universes 
can guarantee that our set of representatives has size at most 
$\cardinality{\representatives} \leq \harmonic(\nconfigs)\ncover$, 
 where $\nconfigs \coloneq \lvert \multiverse\rvert$
 is the size of our multiverse, 
 $\harmonic(\nconfigs)\in\BO(\log\nconfigs)$ denotes the \nconfigs-th harmonic number, 
 and \ncover is the minimum number of representatives needed to cover
 all configurations in~\multiverse. 

 \section{Extended Background}
\label{apx:Background}

This section provides some background information on
\emph{latent-space models}, the objects of our variability assessment. 
We provide further details on two categories of latent-space models that are particularly relevant to our
work: 
\emph{generative models} and \emph{representation-learning algorithms}.
However, our framework can be applied to \emph{any} model that uses embeddings.

\subsection{Generative Models}

Generative models, such as those developed
by~\citet{vaswani2023attention} or \citet{goodfellow2014generative}, are at the
forefront of deep-learning research, enabling the synthesis of new data
as well as complex data transformations such as \emph{style transfer}.
They rely on \emph{embeddings}, i.e., learned
low-dimensional representations of data, to drive their generative
capabilities.
The geometric relationships within the context of the latent space are learned by the
model during training, becoming a cornerstone of the model's characteristics as 
a generator. 
Here, we discuss \emph{variational autoencoders} (VAEs) as the class of generative models featuring most prominently in our experiments. 
Originally developed by ~\citep{kingma2022autoencoding},
VAEs are probabilistic models 
that learn a generative distribution 
$\probability(x, z) = \probability(z) \probability(x|z) $, where
 $\probability(z) $ represents a prior distribution over the latent variable
 $ z $, and
 $\probability(x|z) $ 
is the likelihood function responsible for generating the data~$x$ given~$z$.
VAEs are trained with the objective of maximizing a variational lower bound
$ \LVAE(x) $
on the log-likelihood
$ \log \probability(x) $, subject to the inequality
$ \log \probability(x) \geq \LVAE(x) $.
The expression for the variational lower bound of a VAE is
\begin{equation}
  L_{\text{VAE}}(x) = \mathbb{E}_{q(z|x)} [\log p(x|z)] - \text{KL}(q(z|x) || p(z))\;,
\end{equation}
where
$ q(z|x) $ 
stands for an approximate posterior and
$ \text{KL} $ denotes the Kullback–Leibler divergence~\citep{joyce_kullback-leibler_2011}.
Here, the term
$ \mathbb{E}_{q(z|x)} [\log p(x|z)] $ represents the expected log-likelihood of
$ x $ given
$ z $ under the approximate posterior 
$ q(z|x) $, while
$ \text{KL}(q(z|x) || p(z)) $ 
quantifies the divergence between $ q(z|x) $ and the prior $ p(z) $.
The ultimate objective during the training of VAEs is to maximize the expected lower bound $ \mathbb{E}_{p_d(x)} [L_{\text{VAE}}(x)] $, where $ p_d(x) $ is the data distribution.

\subsection{Representation-Learning Algorithms}
 
Representation-learning algorithms leverage salient features of the input data to obtain structure-preserving
representations.
Many (though not all) representation-learning algorithms are
based on the \emph{manifold hypothesis}, which posits that data are sampled from an~(unknown) low-dimensional latent manifold. 
Even in light of a compelling and ongoing line of research that
questions the integrity of the manifold hypothesis for certain
datasets~\citep{rohrscheidt_topological_2023,brown2023verifying,
scoccola_fibered_2023}, the need to embed data into low-dimensional
latent spaces remains a key aspect for both data preprocessing and the
development of generative models.
Though linear methods like Principal Component Analysis (PCA) are incredibly useful in their own right,
nowadays, 
significant emphasis is placed on \emph{non-linear dimensionality-reduction algorithms} (NLDR algorithms) as the most prominent methods for preserving salient geometric relationships.
These algorithms are maps $\fred\colon \reals^\Dimension \to \reals^\dimension$, typically with $\dimension \ll \Dimension$, that 
make use of local geometric information in a high-dimensional space to 
estimate properties of the underlying manifold.
When mapping into the latent space, such algorithms often aim to maintain \emph{pairwise geodesic
distances} between neighboring points.
Some of the most widely used NLDR algorithms include \texttt{UMAP} by~\cite{mcinnes_umap_2020}, 
\texttt{PHATE} by~\cite{moon_visualizing_2019}, \texttt{t-SNE} by~\cite{maaten_visualizing_2008},
\texttt{Isomap} by~\cite{tenenbaum_global_2000}, and \texttt{LLE} by~\cite{roweis_nonlinear_2000}.

Although the precise behavior of the mapping is unique to individual
implementations, many algorithms rely on estimating local properties of
the data via \nneighbors-nearest-neighbor graphs, treating~\nneighbors
as the \emph{locality scale}, also known as the \texttt{n-neighbors}
parameter.
In combination with the other hyperparameters~(if any), this helps the
algorithm obtain a final representation of the crucial geometric
relationships in the data.
For unsupervised tasks, this raises a non-trivial decision:
\emph{What is the correct scale at which to probe a particular dataset?}
In the absence of labels, the answer to this question is highly context-specific: 
The most insightful geometric relationships cannot be known \emph{a priori}.
Here, we can use \ourmethod to understand the multiverse of representations
that arise from the various algorithmic choices, implementation 
choices, and data choices involved in non-linear dimensionality reduction.  
As demonstrated for VAEs and transformers in the main text,
\ourmethod can describe the distribution of embeddings arising from multiverse considerations and quantify representational variability in NLDR. 
See \cref{apx:DimRed} for 
supplementary experiments on hyperparameter sensitivity in NLDR algorithms and clustering of embeddings across different algorithmic, hyperparameter, and data choices.
 \section{Extended Related Work}
\label{apx:Related Work}

In addition to the related work mentioned in the main paper, which directly
deals with representational variability in one way or the other, our analyses and definitions draw upon a wealth of additional
research in \emph{topological data analysis}. Here, seminal works by
\citet{bubenik2015statistical} and 
\citet{adams2017images} introduce
\emph{persistence landscapes} and \emph{persistence images},
respectively, opening the door toward more efficient topological
descriptors that can be gainfully deployed in a machine-learning
setting. We focus on \emph{persistence landscapes} in this paper since
they do not require any additional parameter choices, but our pipeline
remains valid for \emph{persistence images}. Persistence landscapes have 
the advantage that more of their statistical behavior has been
studied~\citep{bubenik2017persistence}, and recent work even shows that
they capture certain geometric properties of
spaces~\citep{bubenik_persistent_2020}. 
This consolidates their
appeal as an expressive shape descriptor of data.

Beyond research in topological data analysis, several works address the backbone of our pipeline, i.e.,
\emph{persistent homology}, directly. \citet{Cohen-Steiner07a} prove the
seminal stability theorem upon which most of the follow-up work is
based, and \citet{chazal2015subsampling} establish the foundation for
understanding the behavior of topological descriptors~(including
persistence landscapes) under subsampling.   Furthermore, \citet{Chazal14a} show that all
geometric constructions like the Vietoris--Rips complex lead to stable
outcomes in the sense that geometric variation always provides an
upper bound on topological variation. 
We will subsequently make use of this seminal result to motivate the stability and choice of
metrics.

Finally, topological
approaches have also shown their utility in the context of studying individual neural-network models. 
Of particular interest
in current research are the investigation of representational
similarities~\citep{klabunde2023similarity} or the analysis of
particular parts of a larger model, such as attention
matrices~\citep{smith2023parallax}. This strand of research is motivated
by insights into how understanding geometrical-topological
characteristics of data and models can lead to improvements in certain
tasks~\citep{merwe2022manifold, Chen19a}---or, specifically, how the
topology of data can be used to characterize the \emph{loss landscape} of
a model~\citep{freeman2017topology, Horoi22a}.
\vspace*{-12pt}

\end{document}